\documentclass[11pt]{article}
\usepackage[margin=1in]{geometry}
\usepackage{amsmath, amssymb, amsthm}
\usepackage{hyperref}
\usepackage{times}
\usepackage{xcolor}
\usepackage{natbib}  % If citations are needed
\usepackage{graphicx}
\usepackage{authblk} % Useful if authors are included later
\usepackage[ruled]{algorithm2e} % For algorithms

\newtheorem{theorem}{Theorem}
\newtheorem{lemma}{Lemma}
\newtheorem{assumption}{Assumption}
\newtheorem{definition}{Definition}

\newcommand{\E}{\mathbb{E}}
\title{Thompson Sampling for Repeated Newsvendor}
\author{Li Chen, Hanzhang Qin, Yunbei Xu, Ruihao Zhu, Weizhou Zhang}

\date{}

\begin{document}

\maketitle

\begin{abstract}
In this paper, we investigate the performance of Thompson Sampling (TS) for online learning with censored feedback, focusing primarily on the classic repeated newsvendor model--a foundational framework in inventory management--and demonstrating how our techniques can be naturally extended to a broader class of problems. We first model demand using a Weibull distribution and initialize TS with a Gamma prior to dynamically adjust order quantities. Our analysis establishes optimal (up to logarithmic factors) frequentist regret bounds for TS without imposing restrictive prior assumptions. More importantly, it yields novel and highly interpretable insights on how TS addresses the exploration-exploitation trade-off in the repeated newsvendor setting. Specifically, our results show that when past order quantities are sufficiently large to overcome censoring, TS accurately estimates the unknown demand parameters, leading to near-optimal ordering decisions. Conversely, when past orders are relatively small, TS automatically increases future order quantities to gather additional demand information. Then, we extend our analysis to general parametric distribution family and provide proof for Bayesian regret. Extensive numerical simulations further demonstrate that TS outperforms more conservative and widely-used approaches such as online convex optimization, upper confidence bounds, and myopic Bayesian dynamic programming. 
% This study also lays the foundation for exploring general online learning problems with censored feedback.
\end{abstract}

\section{Introduction}
The repeated newsvendor problem is a classic framework in the operations management literature \citep{huh2009nonparametric,besbes2022exploration}. In this problem, a decision-maker must repeatedly chooses how much quantity to stock in each period without knowing the true demand distribution. After each period, the decision-maker observes only censored feedback. That is, the decision-maker only sees how many units were sold (up to the stocking level) but do not learn whether additional demand went unmet once the inventory ran out. %This setup naturally creates a trade-off between exploration (ordering more to learn about the higher end of demand) and exploitation (ordering to minimize immediate costs or lost-sales). %Also, demand pattern change rapidly, the decision maker has to adapt their inventory decisions  to balance learning demand trends and optimizing inventory levels in real time. 
There is a trade-off inherent from this problem between exploration and exploitation: 
\begin{enumerate}
    \item Exploration: stocking more inventory than necessary to gather more information about tail distribution of demand. However doing so may cause the problem of overstocking and incur more holding cost at warehouse. 
    \item Exploitation: order the quantity based on the current estimation of demand so as to minimize the holding cost but doing so may incur lost-sales penalty and fail to gather valuable information of demand distribution, which can cause suboptimal inventory decision in the future. 
\end{enumerate}
More broadly, the repeated newsvendor problem serves as a key representative of a broader class of problems referred to as {\bf ``online learning with censored feedback.''} In this setting, the observation is always the minimum of an unknown random variable and the chosen action. For instance, in the newsvendor problem, the censored feedback corresponds to the minimum of the demand and the stock order. Similarly, in an auction, it is given by the minimum of the buyer's willingness to pay and the seller's set price. These problems inherently exhibit a trade-off between {\bf ``large exploration''}—choosing a sufficiently large action to better observe demand or willingness to pay for more accurate estimation—and {\bf ``optimal exploitation''}—making the right decision to minimize regret. While this paper primarily focuses on the repeated newsvendor problem, we also take an initial step toward systematically exploring the broader class of online learning problems with censored feedback. Existing studies on the repeated newsvendor problem have established a 
$\sqrt{T}$-regret bound under fairly general unknown demand distributions and censored feedback, often leveraging the online convex optimization (OCO) framework \cite{huh2009nonparametric}. However, as widely recognized in the bandit and online learning literature \cite{chapelle2011empirical, seldin2014one, xu2023bayesian}, TS often outperforms OCO-based approaches (which were originally developed for adversarial online learning) as well as other methods such as the Upper Confidence Bound (UCB). These advantages are supported by extensive numerical experiments and theoretical analyses in the aforementioned studies, as well as in our own work. This motivates us to adopt TS as the preferred approach for the repeated newsvendor problem and beyond.

\subsection{Main Contributions}\label{subsec:main messages}
In this paper, we investigate an online learning problem with censored feedback using the classic newsvendor model--one of the most fundamental frameworks in inventory management--as our pivotal example. Specifically, we consider a setting where the true demand $D_t$	
  is unknown, the action $y_t$
  is the order quantity, and the observation $Y_t$	
  is censored feedback given by $ Y_t=\min\{D_t, y_t\}.$
  % \begin{align}\label{eq:censored demand}
  %     Y_t=\min\{D_t, y_t\}.
  % \end{align}
Where demand is exactly observed when sales are less than the order quantity, that is, when $D_t<y_t$; and the demand is censored at the order quantity when sales equal $y_t$, that is, when $D_t \geq y_t$. The newsvendor setting experiences censored feedback--the decision-maker never observes lost-sales if demand exceeds the order quantity. This makes it difficult to accurately estimate demand, as it requires finding the right balance between not ordering too much to prevent excess inventory and placing larger orders to better understand how much demand is actually being missed (i.e., the afore-mentioned exploration-exploitation trade-off). To address this trade-off, One the one side, Our analysis focuses on Bayesian regret analysis for general demand distribution. On the other side, We also establish frequentist guarantees under Weibull distribution--a flexible and widely used parametric family. Our key insights include:
%This structure gives rise to a natural exploration–exploitation trade-off:
%\begin{itemize}
 %     \item On the one hand, choosing large actions $y_t$ helps gather information about the tail of the demand distribution, improving parameter estimation—analogous to uniform exploration in bandit problems.  where sampling different action is helpful to estimating expected payoffs. 
  %\item On the other hand, choosing the right action to minimize immediate regret remains crucial, since overshooting the demand unnecessarily could be costly.
%\end{itemize}
\subsubsection{Frequentist under Weibull Distribution}

\paragraph{\textbf{Estimation under Censored Feedback}}
Regardless of the algorithm used, we derive the confidence interval for demand estimation under censored feedback $Y_t$. The estimation error at round $t$ scales inversely with $\sum_{i=1}^{t-1}(1 - e^{\theta_{\star} y_i^k})$, where $\theta_{\star}$ and $k$ are the scale and shape parameters of the Weibull distribution. This provides a rigorous quantification of how smaller past actions lead to larger errors and highlights the critical trade-off between large exploration and optimal exploitation.

\paragraph{\textbf{Automatic Compensation via TS}}  
From the closed-form expression of TS under Weibull demand, we derive a key insight:
\begin{itemize}
\item When past actions (order quantities) are sufficiently {\it large}, the observed data is more likely to be uncensored and can provide accurate information about the demand’s upper tail. This enables precise estimation of the Weibull parameters and near-optimal ordering decisions.

\item When past actions are relatively {\it small}, TS naturally pushes future actions higher, preventing the algorithm from being stuck with poor estimates. This ensures systematic exploration of larger actions to refine demand knowledge and improve future decisions.
\end{itemize}
In essence, {\it large actions enhance estimation accuracy, while small actions drive future TS-selected actions higher}. This adaptive mechanism allows TS to balance learning and cost minimization, avoiding suboptimal ordering.

\paragraph{\textbf{Balancing Exploration and Exploitation in a Frequentist Setting}}
Despite the Bayesian flavor of TS, we show a frequentist regret bound. In particular, with an  initialization of Gamma prior on the Weibull parameters, TS implicitly achieves the balance between exploration and exploitation, even when we have no prior knowledge of the actual demand parameters. This balance arises because TS automatically change its exploration strategy according to its level of uncertainty in its posterior estimates. As more data is observed, TS naturally puts more weight toward exploitation, which improves estimation accuracy while still allows for occasional exploration to occurs.
\subsubsection{Bayesian Regret under General Parametric Families}
Beyond the Weibull demand, we extend our analysis to general parametric family. To handle censored observations, we estimate the true demand CDF using the  Kaplan–Meier (KM) estimator and construct a plug-in estimator that selects the closest parametric fit to the  KM estimator. Those allows us to establish the a upper confidence interval for Bayesian estimator. Finally we  bound the bayesian regret through Lipschitz property of the newsvendor cost.

\subsubsection{Empirical Effectiveness}
We conduct extensive numerical experiments demonstrating that TS yields competitive performance in terms of cumulative regret, outperforming existing widely-used approaches such as OCO, UCB, and myopic Bayesian dynamic programming. These experiments confirm the widely recognized belief on the effectiveness of TS in online learning and bandit literature and practical applications \cite{chapelle2011empirical, seldin2014one, xu2023bayesian}.

% \subsubsection{Extensions to Online Learning with Censored Feedback}
% We illustrate how our analytical framework naturally extends to broader settings of online learning with censored feedback, making it applicable to a wide range of problems where the feedback is censored. As validated by Assumptions \ref{assumption:regret} and \ref{assumption:action} in Section \ref{sec:extension}, this extension requires only that the regret is Lipschitz continuous with respect to actions and that the relationship between the optimal action and the underlying parameter is continuous and monotone. We also discuss the technical limitation and possible refinement of these assumptions.

\subsection{Related Work}\label{sec:related work}

\paragraph{\textbf{Bayesian Dynamic Programming Literature}}  
The first stream of research formulates this problem using an offline dynamic programming (DP) approach, typically solved via backward induction. However, backward induction often suffers from the curse of dimensionality, making it computationally intractable for large-scale problems. Consequently, existing literature focuses on heuristic solutions as approximations to the optimal policy. \cite{chen2010bounds} propose heuristics based on bounds of Bayesian DP-optimal decisions. Another policy that Bayesian DP literature adopts is a myopic policy, where the decision-maker optimizes inventory decisions one period at a time. This myopic approach has been widely studied in inventory management (see \cite{kamath2002bayesian}, \cite{dehoratius2008retail}, \cite{bisi2011censored}, \cite{besbes2022exploration}, \cite{chuang2023bayesian}). Our approach differs by benchmarking against the ground truth policy, whereas Bayesian DP-based approaches compare against Bayesian DP-optimal policies. In the frequentist setting, the ground truth policy corresponds to the true demand parameter ($\theta^*$). This distinction in benchmarking leads to a fundamentally different regret characterization. Unlike Bayesian DP policies, our method ensures that the regret bound scales as ($\sqrt{T}$). Compared to offline Bayesian DP methods, our work employs TS to learn the unknown demand parameter, providing a simpler and more computationally efficient alternative. Instead of requiring full-state space formulation and solving for an optimal policy via backward induction, our approach dynamically learns the demand distribution while simultaneously making inventory decisions. TS offers a practical solution for real-time decision-making, balancing exploration and exploitation without requiring predefined state transitions.

% \paragraph{\textbf{Non-Parametric and Other Related Newsvendor Literature}}  
% Next, we discuss another line of research that focuses on nonparametric methods for solving joint demand estimation and inventory optimization problems. Unlike the Bayesian approach, which relies on a specific parametric demand distribution, this approach does not impose any predefined distributional assumptions on demand. Instead, it estimates demand directly from observed data. For instance, \cite{huh2009nonparametric} proposes non-parametric adaptive policies that generate ordering decisions over time. Similarly, \cite{agrawal2019learning} proposes an updating confidence interval method that employs a phase-based UCB approach for learning and decision-making. Additionally, recent studies have explored the integration of feature-based learning into inventory systems with censored demand. For instance, \cite{ding2024feature} proposes the feature-based adaptive inventory algorithm and the dynamic shrinkage algorithm, which utilize observed demand patterns and additional features to dynamically adjust inventory policies. Meanwhile, \cite{tang2025offline} extends this idea to a pricing problem under censored demand, demonstrating how contextual features can inform pricing strategies in uncertain demand environments.

\paragraph{\textbf{Thompson Sampling Regret Analysis}}  
In this section, we highlight how our TS regret analysis differs from previous approaches, such as those in \cite{russo2014learning} and \cite{russo2016information}. Specifically, we leverage the problem structure to reformulate regret analysis in terms of the convergence of the posterior parameter, providing a more structured and interpretable framework for regret analysis. A key distinction between our work and \cite{russo2014learning} lies in how exploration and exploitation are handled. Unlike UCB-based methods, which construct deterministic confidence intervals to manage the exploration-exploitation trade-off, TS operates in a Bayesian framework, dynamically updating the posterior distribution based on observed data. By sampling from the posterior distribution, TS inherently balances the need to explore suboptimal actions to gather information and the desire to exploit actions that currently appear optimal.
Additionally, our analysis differs from the information-theoretic regret framework of \cite{russo2016information}, which relies on the concept of the information ratio to bound regret. While this approach has been successfully applied to fully observed bandit problems, it is not directly applicable to our setting, where demand is censored. In censored demand environments, the information ratio is difficult to compute due to missing observations on lost-sales. Instead, our analysis uses the specific structural properties of the newsvendor problem with censored demand. Unlike existing methods that focus on confidence-based or information-theoretic approaches, we introduce a novel regret analysis that directly links regret minimization to the convergence of the posterior distribution.

% The rest of the paper is organized as follows: In Section \ref{sec:Preliminaries and Model Setup}, we present the preliminaries and the newsvendor setup. Section \ref{sec: ts alg newsvendor} details the dynamics of the TS algorithm as applied to the newsvendor problem. In Section \ref{sec:regret analysis}, we provide a regret analysis along with a sketch of the proof. Section \ref{sec:numerical} showcases numerical experiments where we evaluate our algorithm against existing methods. Section \ref{sec:extension} discusses the broader applicability of our framework. Finally, Section \ref{sec: conclusion} concludes our work. All proofs supporting our theoretical claims are provided in the Appendix.

\section{Model Setup and the Thompson Sampling Algorithm}
\label{sec:Preliminaries and Model Setup}
In this section, we discuss the repeated newsvendor model setup and the associated TS algorithm in detail. 
\subsection{Repeated Newsvendor Model}
Following the setup by \cite{bisi2011censored} and \cite{chuang2023bayesian},
we consider a Repeated Newsvendor Model in which a retailer sells a single perishable product over a discrete and finite decision horizon. A Bayesian repeated newsvendor Model can be defined as a tuple $(T,f_{\theta_{\star}}(\cdot),\rho_0(\cdot),h,p)$, where $T \in \mathbb{R}^+$ is the known length of decision horizon, $f_{\theta_{\star}}(\cdot)$ is the known class of demand distributions, parameterized by an unknown parameter $\theta_{\star}$. We define the expression of $f_{\theta_{\star}}(\cdot)$ and $\rho_0(\cdot)$ in the next subsection. $h>0$ is the unit overage cost, and $p>0$ is the unit stock-out penalty. $h$ occurs if there is any leftover. $p$ occurs if there is any unmet demand. 
%which is sampled from a known prior distribution $\rho_0(\cdot)$.
%\textcolor{magenta}{LX: Could you better describe what is the exact definition of $h$ and $p$ in a inventory story-telling? If I miss some environment parameter, just describe it here...}
The dynamic is defined as follows. Before the decision-making process, the parameter $\theta_{\star}$ is unknown.
At time $t \in [T]$, three events happen sequentially:
% Consider a retailer selling a single perishable product sold for consecutive $T$ periods. For each time step $t \in [T]$ of the inventory control problem, the following sequence of events happens:
\begin{enumerate}
    \item The retailer determines an order quantity $y_t \geq 0$.
    % At the beginning of each period $t \in [T]$, DM determines an order quantity $y_t \geq 0$.
    \item The demand $D_t$ is i.i.d generated from demand distribution $f_{\theta_{\star}}(\cdot)$.
    % Demands $D_t$ s are i.i.d generated from a demand distribution $f_{\theta}(\cdot)$ parameterized by an \emph{unknown} demand parameter $\theta_{\star}$. We assume that $\theta$ is sampled from a known prior $\rho_0(\theta_{\star})$.
    % \item A unit overage cost $h$ and stock-out penalty $p$ are incurred if there is any leftover or unmet demand, respectively.
    \item Lost-sales are not observed, demand $D_t$ are censored on the right by the inventory levels $y_t$. The retailer only observes the data pairs $\left(Y_t, \delta_t\right)$, where $Y_t =D_t \wedge y_t$ and $\delta_t =1 \left[D_t <  y_t \right]$. interpreted as the number of exact observations of demand. Where demand is exactly observed when sales are less than the order quantity, that is, when $D_t<y_t$; and the demand is censored at the order quantity when sales equal $y_t$, that is, when $D_t \geq y_t$.
    
    The expected cost incurred at time step $t$ is   
    \begin{equation}
     \label{cost_function}
        g(y_t,D_t)=\E\left[h\left(y_t-D_t\right)^{+}+p\left(D_t-y_t\right)^{+}\right].
    \end{equation}
\end{enumerate}
The retailer knows the length of horizon $T$, the class of demand distributions $f_{\theta}(\cdot)$, the prior distribution $\rho_0$, $h$ and $p$, but does not know the exact value of $\theta_{\star}$.

According to \cite{chuang2023bayesian}.
We denote $H = \left\{H_t \right\}$ the natural filtration generated by the right-censored sales data, i.e $H_t=\sigma  \left\{(Y_i,\delta_i) : i \leq t\right\} $, where $Y_t =D_t \wedge y_t$ and $\delta_t =1 \left[D_t <  y_t \right]$. The retailer chooses an action $y_t$. The retailer aims to minimize the total expected cost in the $T$- period
online phase. We quantify the performance guarantee of the DM's non-anticipatory policy $\pi$ by its regret. We define regret as $\operatorname{Regret(T, \pi,\theta_{\star})}$ be the regret with respect to a fixed $\theta_{\star}$. 

\begin{definition}
\label{def:freq regret}
\begin{align*}
 & \operatorname{Regret(T, \pi,\theta_{\star})}=\mathbb{E}\left[\sum_{t=1}^Tg\left(y_t, D_t\right) - \sum_{t=1}^T g\left(y_{\star}, D_{t}\right)\mid 
 \theta_{\star}\right].  
%\label{eq: TS regret}
% &  %\operatorname{BayesianRegret(T,\pi)}=\mathbb{E}[\operatorname{Regret}(T,\pi,\theta_{*})]
\end{align*}
\end{definition}
\begin{definition}
\label{def:bayeregret}
\begin{align*}
 & \operatorname{BayeRegret(T, \pi)}=\mathbb{E}\left[\sum_{t=1}^Tg\left(y_t, D_t\right) - \sum_{t=1}^T g\left(y_{\star}, D_{t}\right)\right].  
% &  %\operatorname{BayesianRegret(T,\pi)}=\mathbb{E}[\operatorname{Regret}(T,\pi,\theta_{*})]
\end{align*}
\end{definition}
For simplicity, throughout the paper we abbreviate $\operatorname{Regret(T, \pi,\theta_{\star})}$  as $\operatorname{Regret(T,\theta_{\star})}$. 
The optimal order quantity is given by
$
    y_{\star} = \mathop{\arg\max}_y \  \E\left[h\left(y-D_t\right)^{+}+p\left(D_t-y\right)^{+}\right] = F^{-1}_{\theta_{\star}}\left(\frac{p}{p+h} \right).$
which corresponds to the critical quantile of the demand distribution when $\theta_{\star}$ is known.

%and $\operatorname{BayesianRegret(T,\pi)}$  as $\operatorname{Regret(T,\theta_{\star})}$ and $\operatorname{BayesianRegret(T)}$.

%\textcolor{magenta}{LX: Put the definition of Bayesian Regret (Currently in Section 3.2.1) and the definition of $y_*$ (Currently in Section 3.2.2, in English words and mathematical to describe $y_{*}(\theta_{\star})=F^{-1}_{\theta_{\star}}(\frac{p}{p+h})$) in the following:

% We note in definition \ref{def:freq regret}, the regret here is essentially a frequentist (non-Bayesian) regret. In this definition, $\theta_\star$ should be viewed as a fixed parameter.  

% Even though our work focuses on the development of TS, which is a Bayesian online learning algorithm, our regret analysis holds for the more general case in which the prior demand can be drawn from arbitrary probability distributions.

\subsection{Preliminaries}
In this section, we introduce the necessary tools to implement TS.

\vspace{2mm}
\noindent\textbf{Demand Distribution: Newsvendor Family.} The newsvendor (newsboy) family, introduced by \cite{braden1991informational}, is known to be the only family whose posterior distributions with censored demand information have conjugate priors. A random variable is a member of the newsvendor distributions if its density is given by $
 f_{\theta}(x )=\theta d^{\prime}(x) e^{-\theta d(x)}, \quad F_{\theta}(x)=1-e^{-\theta d(x)},
$
where $d^{\prime}(x)>0, \ \forall x>0$, so $f_{\theta}(x )$ is positive on $(0, \infty)$
 $\lim _{x \rightarrow 0} d(x)=0$ and $\lim _{x \rightarrow \infty} d(x)=\infty$. So $F_\theta(x)$ is a valid probability distribution,
where $d(x)$ is a positive, differentiable, and increasing function and $\theta \in R_{+}$. \cite{lariviere1999stalking} show that when the demand distribution is Weibull with a gamma prior, the optimal solution for repeated newsvendor problem admits a closed form. 
% By letting $d(x)=x^k$ with a known constant $k>0$, we get the Weibull distribution. If $k=1$, we get the exponential distribution. In such cases, the underlying density function of demand is $
% f_\theta(x)=\theta k x^{k-1} e^{-\theta x^k}$.

\vspace{2mm}
\noindent\textbf{Prior Distribution and Parametric Demand.} With the true value of $\theta_{\star}$ being unknown, the decision maker initiates TS with a prior distribution $\rho_0$ at the outset. Throughout the paper, we adopt the prior family and parametric demand introduced by \cite{braden1991informational}. Namely, the prior follows $\rho_0 \sim \operatorname{Gamma}(\alpha_0, \beta_0)$ ($\rho_0(\theta) = \frac{\beta_0^{\alpha_0}}{\Gamma(\alpha_0)} \theta^{\alpha_0 - 1} e^{-\theta \beta_0})$. When demand is described by a member of the newsvendor family, the gamma distribution remains a conjugate prior. Under the Weibull distribution of demand, we have $  F^{-1}_{\theta_{\star}}(\frac{p}{p+h})=\left( \frac{1}{\theta_{t}} \right)^{1/k}\left(-\ln (\frac{h}{p+h})\right)^{1 / k},$
% \begin{equation*}
%     F^{-1}_{\theta_{\star}}(\frac{p}{p+h})=\frac{1}{\theta_{\star}}\left(-\ln (\frac{h}{p+h})\right)^{1 / k},
% \end{equation*}
 where $F^{-1}_{\theta_{\star}}$ is the inverse cumulative distribution function. Here, we emphasize that this prior is only used to initiate TS and we do not impose any prior distribution on $\theta_{\star}$.

\vspace{2mm}
\noindent\textbf{Likelihood Function.} The likelihood function can be formulated for a set of observed data pairs, including both censored and uncensored data. Let's start by considering the first censored data pair denoted as $\left(Y_0, \delta_0\right)$. We use  $\theta \mapsto \mathcal{L}\left(\theta \mid Y_0, \delta_0\right)$ to denote the likelihood function:
$$
\mathcal{L}\left(\theta \mid Y_0, \delta_0\right)= \begin{cases}f_\theta\left(Y_0\right), & \text { if } \delta_0=1 ; \\ 1-F_\theta\left(Y_0\right), & \text { if } \delta_0=0 .\end{cases}
$$

 Consider we have $t \in [T]$ observations of data pairs denoted as $Y=(Y_0,Y_1,\cdots,Y_t),\delta=(\delta_0,\cdots,\delta_t)$ and we use $C$ denote the set of all observations of censored data pairs and $\Bar{C}$ denote the set of all observations of uncensored data pairs , where $|C|=m$, $|\Bar{C}|=n$, and $m+n=t$.
 Then the likelihood function is
\begin{align*}
     \mathcal{L}\left(\theta \mid Y, \delta \right) =\Pi_{i=1}^{n} \left(f_\theta\left(Y_i\right)\right)^{i}  \Pi_{j=1}^{m} \left((1-F_\theta\left(Y_j\right)\right)^{j} =\left(\theta k\right)^n \left(\Pi_{i=1}^{n} Y_i\right)^{k-1} e^{-\theta\sum_{l=1}^{t} Y_l^{k}}.
\end{align*}

\vspace{2mm}
\noindent\textbf{Posterior Update.} The posterior demand distribution $\rho_t$ at the beginning of period $t$ can be derived as: 
\begin{align*}
    \rho_t(\theta) & \propto {\rho_0} \times \mathcal{L}\left(\theta \mid Y, \delta \right) \\
    % & \propto \beta_0^{\alpha_0} \theta^{\alpha_0-1} e^{-\theta \beta_0} \times \left(\theta k\right)^n \left(\Pi_{i=1}^{n} Y_i\right)^{k-1} e^{-\theta\sum_{l=1}^{t} Y_l^{k}}\\
    % & \propto  \theta^{\alpha_0+n-1} e^{-\theta \left(\beta_0+\sum_{i=1}^{t} Y_l^{k}\right)}\\  
      & \propto \operatorname{Gamma}(\alpha_0+\sum_{i=1}^{t} \delta_i,\beta_0+\sum_{i=1}^{t} Y_i^{k}).
\end{align*} 
Thus, the posterior at the beginning of period $t$ is given by $\rho_t = \operatorname{Gamma}\left(\alpha_t, \beta_t\right)$, where $\alpha_t=\alpha_0+\sum_{i=1}^{t} \delta_i$ and $\beta_t=\beta_0+\sum_{i=1}^{t} Y_i^k$.
\subsection{Algorithm: Thompson Sampling for Repeated Newsvendor Problem}
\label{sec: ts alg newsvendor}
The TS algorithm is presented in \ref{alg:ts for newsvendor}.  TS is a Bayesian approach used to balance exploration and exploitation in sequential decision-making problems. In the context of the newsvendor problem, TS can be implemented to decide on the optimal order quantity under demand uncertainty. 

\begin{algorithm}[htb]
\caption{TS for Repeated Newsvendor}
\label{alg:ts for newsvendor}
\KwIn{Prior distribution $\rho_0=$$\operatorname{Gamma}(\alpha_{0},\beta_{0})$, where $\alpha_0 \ge \max\left\{ \frac{\ln{\frac{T}{\delta}}}{\ln{\frac{e}{2}}}, 2 \right\},$ $\delta \in \left(0, \frac{1}{6}\right)$, Time Horizon $T.$ }
\For{$t=1$ to $T$}{
 Place order quantity $ y_{t}=\left( \frac{1}{\theta_{t}} \right)^{1/k}\left(-\ln (\frac{h}{p+h})\right)^{1 / k},$
 where $\theta_t \sim \operatorname{Gamma}(\alpha_{t-1},\beta_{t-1})$\;
 Observe sales $Y_t=\min\{D_t,y_t\}$ and indicator of whether demand is censored $\delta_t = \mathbf{1}[D_{t} <y_{t}];$
 Update the posterior $\rho_t \sim \operatorname{Gamma}\left(\alpha_{t}, \beta_{t}\right)$, where $\alpha_{t}=\alpha_{0}+\sum_{i=1}^{t} \delta_{i}, \ \beta_{t}=\beta_{0}+\sum_{i=1}^{t} Y_{i}^k.$
 % \begin{equation*}
 %     \alpha_{t}=\alpha_{0}+\sum_{i=0}^{t-1} \delta_{i}, \quad \beta_{t}=\beta_{0}+\sum_{i=0}^{t-1} Y_{i}^k.
 % \end{equation*}

 % which is given by $\rho_t =\operatorname{Gamma}\left(\alpha_{t}, \beta_{t}\right)$, where $\alpha_{t}=\alpha_{0}+\sum_{i=0}^{t-1} \delta_{i}$ and $\beta_{t}=$ $\beta_{0}+\sum_{i=0}^{t-1} Y_{i}^k$. $Y_{t}=D_{t} \wedge y_{t} $ is the observed value and $\delta_{t}=\mathbf{1}[D_{t} <y_{t}]$, an indicator of whether demand is censored\;
}
\end{algorithm}
%\subsection{Bayesian Regret}
Initially, the environment draws a sample of $\theta_{\star}$ from prior $\rho_{0}=\operatorname{Gamma}(\alpha_{0},\beta_{0})$, which is unknown to retailer.  and a known time horizon $T$. Then, for each $t \in [T]$, retailer places the order quantity $y_t$ and then observes the sales $Y_t$, which is the minimum of demand and order quantity. Then the posterior is updated accordingly.  $y_t$ iteratively updates the posterior and samples from it. Specifically, 
$ y_t=F^{-1}_{\theta_t}(\frac{p}{p+h})
    =\left( \frac{1}{\theta_{t}} \right)^{1/k}\left(-\ln (\frac{h}{p+h})\right)^{1 / k}$ and the property of $\theta_t$.
We sample $\theta_t$ from a Gamma distribution with shape parameter $\alpha_t$ and rate parameter $\beta_t$, i.e., $\theta_t \sim \text{Gamma}(\alpha_t, \beta_t)$. This choice is motivated by conjugate prior properties in the posterior update.
$\mathbb{E}[\frac{1}{\theta_t}] = \frac{\beta_t}{\alpha_t - 1}$.
TS efficiently balances exploration (learning about the true demand distribution) and exploitation (placing optimal orders based on current knowledge). 

% This approach is particularly useful in multi-period inventory problems, where demand is uncertain and needs to be learned over time.

 \section{Frequentist Regret Analysis for Weibull Distribution}
 In this section, we provide the analysis for the regret upper bound on our Algorithm \ref{alg:ts for newsvendor}, which is shown in the following theorem:
 \begin{theorem}
$T$-period regret of a given $\theta_{\star}$ for repeated newsvendor problem is 
  \label{thm:bay-regret-ts}
\begin{align*}
    \operatorname{Regret(T,\theta_{\star})}\leq \tilde{O}\left(\max\{h,p\} \cdot \left(-\ln (\frac{h}{p+h})\right)^{1 / k} \cdot\frac{1}{\theta_{\star}^2}\cdot \sqrt{T}\right).
\end{align*}  
\end{theorem}
% \begin{align*}
%      \tilde{O}\left(\max\{h,p\} \cdot \left(-\ln (\frac{h}{p+h})\right)^{\frac{1}{k}}\cdot \frac{1}{\theta_{\star}^2} \cdot \sqrt{T}\right).
% \end{align*} 
We provide a sketch proof for proving the Theorem \ref{thm:bay-regret-ts}. The proof consists of three key steps: \textbf{1. Lipchitz Continuity of Regret }(Section \ref{subsubsec:Lipchitz regret}), \textbf{2. Confidence Analysis of Estimation } (Section \ref{subsubsec:confidence}), and \textbf{ 3. Lower bounding the Actions} (Section \ref{subsubsec:lower bound action}). We also discuss how the these steps can be generalized to broader models of online learning with censored feedback in Section \ref{sec:extension}.
\subsection{\textbf{Key Step 1: Uniform Lower Bound of $y_t$}}\label{subsubsec:lower bound action}

In order to establish the regret bound, it is essential to first prove a uniform lower bound for $y_t$, as this will play a crucial role in subsequent derivations. According to Lemma \ref{lower bound of y_t}, we have:

\begin{lemma}
\label{lower bound of y_t}
$$\mathbb{P}\left(\frac{1}{\theta_t} > \frac{\beta_t}{2\alpha_t}\right) \ge 1-\left(\frac{2}{e}\right)^{\alpha_t},\qquad \forall \  t \in [T].$$
\end{lemma}

The proof of Lemma \ref{lower bound of y_t} is provided in Appendix \ref{appendix: lower bound of y_t}. Building upon this lemma, we proceed by conditioning on the event that $\frac{1}{\theta_t} > \frac{\beta_t}{2\alpha_t}$ and that the demand $D_t$ satisfies $D_t \ge \underline{D}$. Under these conditions, we can derive a lower bound for $y_t$ as follows:
\begin{subequations}
    \begin{align}
    y_t &= \left( \frac{1}{\theta_{t}} \right )^{1/k}\left(-\ln (\frac{h}{p+h})\right)^{1/k}  \nonumber\\
    & \ge \left( \frac{\beta_t}{2\alpha_t}\right)^{1/k} \cdot \left(-\ln (\frac{h}{p+h})\right)^{1 / k} \label{lb-yt-a}\\
    % & \ge \frac{\beta_t}{2\alpha_t} \cdot \left(-\ln (\frac{h}{p+h})\right)^{1 / k} \label{lb-yt-b}\\
    &  = \left(\frac{1}{2} \right)^{1/k} \left(-\ln (\frac{h}{p+h})\right)^{1 / k} \cdot \left (\frac{\beta_0+\sum_{i=1}^{t} \min\{y_i,D_i\}^k}{\alpha_0+\sum_{i=1}^{t} \delta_i} \right )^{1/k}\label{lb-yt-c} \\
    & \ge \left(- \frac{1}{2}\ln (\frac{h}{p+h})\right)^{1 / k} \cdot \min \left \{\left(\frac{\beta_0}{\alpha_0} \right)^{1/k},\underline{D} \right\}  = L.\label{lb-yt-d}
    \end{align}
\end{subequations}

% Here, equation (\ref{lb-yt-a}) follows directly from the definition of $y_t$ as given in equation (\ref{eq: yt}). Inequality (\ref{lb-yt-b}) utilizes the result from Lemma \ref{lower bound of y_t}, indicating that with high probability, $\frac{1}{\theta_t}^{\frac{1}{k}}$ is bounded below by $\frac{\beta_t}{2\alpha_t}$. The equality in (\ref{lb-yt-c}) comes from the update rules for $\alpha_t$ and $\beta_t$ as defined in Algorithm \ref{alg:ts for newsvendor}. Finally, inequality (\ref{lb-yt-d}) is justified by applying Lemma \ref{lemma:sequence a and b}, which is an auxiliary result crucial to our analysis. 

Here, equation~(\ref{lb-yt-a}) follows directly from the definition of $y_t$ in equation~(\ref{eq: yt}) together with Lemma~\ref{lower bound of y_t}, which guarantees that $\left(1/\theta_t\right)^{1/k}$ is, with high probability, bounded below by $\left(\beta_t/(2\alpha_t)\right)^{1/k}$. 
Equality~(\ref{lb-yt-c}) then uses the update rules for $\alpha_t$ and $\beta_t$ in Algorithm~\ref{alg:ts for newsvendor}. 
Finally, inequality~(\ref{lb-yt-d}) is obtained by applying Lemma~\ref{lemma:sequence a and b}, which ensures the ratio is uniformly bounded from below, thereby establishing the constant lower bound $L$.

\begin{lemma}
\label{lemma:sequence a and b}
    for two sequence $\{a_i\}_{i=1}^n$, $\{b_i\}_{i=1}^n$ satisfies $a_i \ge 0$ and $b_i \ge 0$ for any $i \in [n]$, and for at least one $i \in [n]$, $b_i > 0$. Then we have
    \begin{equation*}
        \frac{\sum_{i=1}^n a_i}{ \sum_{i=1}^n b_i} \ge \min_{i \in [n]: b_i > 0} \left \{ \frac{a_i}{b_i}\right \}.
    \end{equation*}
\end{lemma}

%%% talk about lemma 3.7 when delta_0=1 observed the d 
%when delta_1=0, one of the most important ascept of TS algorithm !!!
% fill in insight 
%delete related work at least one page (one and half) 

From the closed-form expression \eqref{lb-yt-c} and Lemma \ref{lemma:sequence a and b}, we reveal the most important ascept of TS algorithm as follows: 
\begin{enumerate}
    \item when $\delta_i=1$ (i.e. $D_t < y_t$), we obtain the full observation demand. the increment is $\frac{D_t}{1}$ As a result, the observed data is uncensored and can provide accurate information about the demand’s upper tail. 
    \item when $\delta_i=0$ (i.e. $D_t \ge y_t$). we get the censored demand, which indicates the past action is relatively small. Interestingly, since $\delta_i$ appears in the denominator in the closed-form expression \eqref{lb-yt-c}, TS naturally pushes future actions higher in subsequent periods, preventing the algorithm from getting stuck with poor estimates.
\end{enumerate}
This key observation illustrates how TS automatically balances the exploration-exploitation trade-off in the repeated newsvendor problem.
 
Applying Lemma \ref{lemma:sequence a and b} ( proved in Appendix \ref{appendix: proof sequence a and b} ) in our context, and considering that $D_t \ge \underline{D}$, we conclude that $y_t$ is uniformly bounded below by $L$ for all $t$.

This uniform lower bound on $y_t$ is a critical to establish the regret bound. we plug back the lower bound to the definition of $\alpha_t $ in Lemma \ref{lemma:alpha_t,beta_t} to analyze the term $\left|\alpha_t-1\right|$, which analysis is referred in Appendix \ref{appendix: proof for lemma: alpha_t-1}.
From the closed-form expression \eqref{lb-yt-d} and Lemma \ref{lemma:sequence a and b}, we reveal the most important aspect of the TS algorithm as follows: 
\begin{enumerate}
    \item When $\delta_i=1$ (i.e. $D_t < y_t$), we obtain the full demand observation. The increment is uncensored and can provide accurate information about the demand’s upper tail. 
    \item When $\delta_i=0$ (i.e. $D_t \ge y_t$), we get censored demand, which indicates the past action was relatively small. Interestingly, since $\delta_i$ appears in the denominator in \eqref{lb-yt-c}, TS naturally pushes future actions higher in subsequent periods, preventing the algorithm from getting stuck with poor estimates.
\end{enumerate}

This key observation illustrates how TS automatically balances the exploration-exploitation trade-off in the repeated newsvendor problem. Applying Lemma \ref{lemma:sequence a and b} (proved in Appendix \ref{appendix: proof sequence a and b}) in our context, and considering that $D_t \ge \underline{D}$, we conclude that $y_t$ is uniformly bounded below by $L$ for all $t$. This uniform lower bound on $y_t$ is critical for establishing the regret bound.

\subsection{\textbf{Key Step 2: Confidence Analysis of $\left|\E[y_t]-y_{\star}\right|$}}
\label{subsubsec:confidence}

Before proceeding, we give the definition for $y_t$ and $y_{\star}$:

\begin{lemma}
\label{lem:y_t-and-y_star}
The order quantity $y_t$ at $t$ and the optimal myopic order quantity $y_{\star}$ satisfy
\begin{align}
\label{eq: ystar}
     y_{\star}(\theta_{\star})&=F^{-1}_{\theta_{\star}}(\tfrac{p}{p+h})= \left ( \frac{1}{\theta_{\star}} \right )^{1/k}\left(-\ln (\tfrac{h}{p+h})\right)^{1 / k}\\
\label{eq: yt}
   y_t(\theta_{t})&=F^{-1}_{\theta_t}(\tfrac{p}{p+h})
   =\left( \frac{1}{\theta_{t}} \right)^{1/k}\left(-\ln (\tfrac{h}{p+h})\right)^{1 / k}
\end{align}
where $F^{-1}$ is the inverse cumulative distribution function of the demand distribution. Moreover, 
  $\E\left[\tfrac{1}{\theta_t}\right]=\tfrac{\beta_t}{\alpha_t-1}$.
\end{lemma}

By examining equations (\ref{eq: ystar}) and (\ref{eq: yt}), we obtain:
\begin{align}\label{eq:lipchitz action 2}
    \begin{aligned}
      \left|\E \left[y_t\right]-y_{\star}\right| & = \left(-\ln (\frac{h}{p+h})\right)^{1 / k} \left|\mathbb{E} \left[ \left( \frac{1}{\theta_t} \right)^{1/k} \right ]- \left( \frac{1}{\theta_{\star}} \right)^{1/k} \right| \\
        & = \left(-\ln (\frac{h}{p+h})\right)^{1 / k} \left|\mathbb{E} \left[ \left( \frac{1}{\theta_t} \right)^{1/k} - \left( \frac{1}{\theta_{\star}} \right)^{1/k} \right ] \right| \\
        & \le \left(-\ln (\frac{h}{p+h})\right)^{1 / k} \mathbb{E} \left[ \left | \left( \frac{1}{\theta_t} \right)^{1/k} - \left( \frac{1}{\theta_{\star}} \right)^{1/k}  \right | \right ]  \\
        & \le \left(-\ln (\frac{h}{p+h})\right)^{1 / k} \cdot \frac{1}{k} \cdot \min \left \{L, \frac{1}{\theta_*} \right \}^{1/k - 1} \cdot \mathbb{E} \left[ \left | \left( \frac{1}{\theta_t} \right) - \left( \frac{1}{\theta_{\star}} \right)  \right | \right ]
    \end{aligned}
\end{align}

The first inequality comes from jensen inequality $|\mathbb{E} [X]| \le \mathbb{E}[|X|]$. The second inequality comes from Lipschitz inequality and $y_t$'s lower bound stated earlier in Lemma \ref{lower bound of y_t}.

To proceed further, we establish a range for the demand $D_t$ at each time $t$. The following lemma provides this range with high probability:

\begin{lemma}
\label{lem:demand-range}
    For each $t \in [T]$, with probability $\ge 1-\delta/T$, the realization of demand $D_t \sim \operatorname{Weibull}(\theta_{\star})$ lies in
    \begin{align*}
     \underline{D}=\left(\tfrac{\ln{\left(\tfrac{2T}{2T-\delta}\right)}}{\theta_{\star}}\right)^{\frac{1}{k}}, \qquad 
     \overline{D}=\left(\tfrac{\ln{\left(\tfrac{2T}{\delta}\right)}}{\theta_{\star}}\right)^{\frac{1}{k}}.
 \end{align*}
\end{lemma}
Lemma \ref{lem:demand-range} is proved in Appendix \ref{appendix: lemma demand range}. 
This lemma ensures that, with high probability, the demand realizations are confined within the specified range, which is crucial for later analysis. Next, we provide confidence bound for how close the $\frac{1}{\theta_t}$ and its mean $\frac{\beta_t}{\alpha_t-1}$ is. Ideally, as $t$ increases, $\theta_t$ will converge to $\theta_*$ and $\frac{1}{\theta_t}$ will converge to $\frac{\beta_t}{\alpha_t-1}$. The following lemma shows the rate of convergence as follows: 
\begin{lemma}
\label{beta/alpha-1}
For any $t \in [T]$ and for any realization of  $\theta_{\star}$, 
\begin{align*}
    \mathbb{P}\left( \left|\tfrac{\beta_t}{\alpha_t-1}- \tfrac{1}{\theta_{\star}}\right|  \geq \sqrt{\ln{\left(\tfrac{2t^2}{\delta}\right)}}\left(\overline{D}^k+\tfrac{2}{\theta_{\star}}\right)\sqrt{\tfrac{t}{\left(\alpha_t-1\right)^2}} \right) \leq \tfrac{\delta}{t^2}.
 \end{align*}    
\end{lemma}
Lemma \ref{beta/alpha-1} is proved in Appendix \ref{sec:appendix-lemma beta/alpha-1}. This lemma provides a probabilistic bound on the estimation error of $\frac{1}{\theta_t}$, which is key in assessing the accuracy of the order quantity decisions over time.
Combining Lemmas \ref{lem:demand-range} and \ref{beta/alpha-1}, we bound
\begin{equation*}
    \begin{aligned}
     \left|\E \left[y_t\right]-y_{\star}\right| 
     & \leq \left(-\ln (\frac{h}{p+h})\right)^{1 / k}\cdot \frac{1}{k} \cdot \min \left \{L, \frac{1}{\theta_*} \right \}^{1/k - 1} \sqrt{\ln{\left(\frac{2t^2}{\delta}\right)}}\left(\overline{D}^k+\frac{2}{\theta_{\star}}\right)\sqrt{\frac{t}{\left(\alpha_t-1\right)^2}}\\
& \leq \left(-\ln (\frac{h}{p+h})\right)^{1 / k}\left(\overline{D}^k+\frac{2}{\theta_{\star}}\right)\cdot \frac{1}{k} \cdot \min \left \{L, \frac{1}{\theta_*} \right \}^{1/k - 1} \sqrt{2\ln{\left(\frac{T}{\delta}\right)}}\sqrt{\frac{t}{\left(\alpha_t-1\right)^2}}.
\end{aligned}
\label{eq:sket-pf-bound-yt-decompose}
\end{equation*}
\begin{lemma}[\cite{chuang2023bayesian}]
\label{lemma:alpha_t,beta_t}
    The stochastic processes $\left\{\alpha_t\right\}$ and $\left\{\beta_t\right\}$ can be represented by
\begin{align*}
\alpha_t & =\alpha_0+\sum_{i=0}^{t-1} \delta_i =\alpha_0+\sum_{i=0}^{t-1} \mathbb{E}_{\theta_{\star}}\left[\delta_i \mid H_{i-1}\right]+\sum_{i=0}^{t-1}\left(\delta_i-\mathbb{E}_{\theta_\star}\left[\delta_i \mid H_{i-1}\right]\right)=\alpha_0+\sum_{i=0}^{t-1}\left(1-e^{-\theta_{\star} y_i^k}\right)+M_t
\end{align*}
where $M_t =\sum_{i=0}^{t-1}\left(\delta_i-\mathbb{E}\left[\delta_i \mid \mathcal{H}_{i-1}\right]\right)-1$.

\end{lemma}
From the above lemma \ref{lemma:alpha_t,beta_t}, we can see that as long as $y_t$ has a lower bound, we are able to derive the upper bound for regret. Then,
we plug back the lower bound to the definition of $\alpha_t $ in Lemma \ref{lemma:alpha_t,beta_t} to analyze the term $\left|\alpha_t-1\right|$, which analysis is referred in Appendix \ref{appendix: proof for lemma: alpha_t-1}. To establish the regret, we use the technique of truncating $T$ as follows. Denote 
$T_0 = 64\left(1-\exp\{-\theta_{\star} L^k\}\right)^{-2}\ln{\left(\frac{T}{\delta}\right)}.
$
According to the proof of Lemma 11 in the Appendix , we have $\forall t\in [T],$
$$
    \frac{\sqrt{t}}{\alpha_t -1} \le
    \begin{cases}
        \frac{\sqrt{T_0}}{\alpha_0 - 1} & t \le T_0, \\
        \frac{\sqrt{t}}{\frac{1}{2} t(1-\exp{\{-\theta_{\star} L^k\}})} = \frac{2}{\sqrt{t} \cdot (1-\exp{\{-\theta_{\star} L^k\}})} & t > T_0.
    \end{cases}
$$
Therefore,
$$
 \sum_{t=1}^T \frac{\sqrt{t}}{\alpha_t - 1} = \sum_{t=1}^{T_0} \frac{\sqrt{t}}{\alpha_t - 1} + \sum_{t=T_0}^T \frac{\sqrt{t}}{\alpha_t - 1} \\
         \le T_0 \cdot \frac{\sqrt{T_0}}{\alpha_0 - 1}  + \frac{2}{(1-\exp{\{-\theta_{\star} L^k\}})} \cdot \sum_{t=1}^T \frac{1}{\sqrt{t}} \\
         \le T_0^{\frac{3}{2}} \cdot \frac{1}{\alpha_0 - 1} + 4\left(1-\exp\{-\theta_{\star} L^k\}\right)^{-1}\sqrt{T}
$$
% \begin{equation*}
%     \frac{\sqrt{t}}{\alpha_t -1} \le 
%     \begin{cases}
%         \frac{\sqrt{T_0}}{\alpha_0 - 1} & t \le T_0, \\
%         \frac{\sqrt{t}}{\frac{1}{2} t(1-\exp{\{-\theta_* L^k\}})} = \frac{2}{\sqrt{t} \cdot (1-\exp{\{-\theta_* L^k\}})} & t > T_0.
%     \end{cases}
% \end{equation*}
% See page 21 bottom. Therefore,
% \begin{equation*}
%     \begin{aligned}
%         \sum_{t=1}^T \frac{\sqrt{t}}{\alpha_t - 1} &= \sum_{t=1}^{T_0} \frac{\sqrt{t}}{\alpha_t - 1} + \sum_{t=T_0}^T \frac{\sqrt{t}}{\alpha_t - 1} \\
%         & \le T_0 \cdot \frac{\sqrt{T_0}}{\alpha_0 - 1}  + \frac{2}{(1-\exp{\{-\theta_* L^k\}})} \cdot \sum_{t=1}^T \frac{1}{\sqrt{t}} \\
%         & \le T_0^{\frac{3}{2}} \cdot \frac{1}{\alpha_0 - 1} + 4\left(1-\exp\{-\theta_{\star} L^k\}\right)^{-1}\sqrt{T}
%     \end{aligned}
% \end{equation*}
% Therefore
% \begin{equation*}
% \begin{aligned}

\subsubsection{\textbf{Key Step 3: Regret Decomposition via Lipschitz Continuity}}\label{subsubsec:Lipchitz regret}

We decompose the $\operatorname{Regret(T,\theta_{\star})}$ as follows: By the Lipchitz continuity of $\min$, 
\begin{subequations}
    \begin{align}
\operatorname{Regret(T,\theta_{\star})} & =\mathbb{E}\left[\left(\sum_{t=1}^Tg\left(y_t, D_t\right)-\sum_{t=1}^T pD_t\right) - \left(\sum_{t=1}^T g\left(y_{\star}, D_{t}\right)-\sum_{t=1}^T pD_t\right)\right] \nonumber\\ 
&= \E\left[ \sum_{t=1}^T\left[hy_t-(h+p)\min\{y_t,D_t\}\right]-\sum_{t=1}^T\left[hy_{\star}-(h+p)\min\{y_{\star},D_t\}\right]\right] \nonumber\\
&=\E\left[ \sum_{t=1}^T\left[h\left(y_t-y_{\star}\right)\right]-\sum_{t=1}^T(h+p) \left(\min\{y_t,D_t\}-\min\{y_{\star},D_t\}\right)\right] \nonumber\\
&\leq  \max\{h,p\} \cdot  \E \left[\sum_{t=1}^T\left|\E\left[y_t\right]-y_{\star}\right|\right] \label{pf-regret-a-sketch} \\
& = \max\{h,p\} \cdot \sum_{t=1}^T \E \left[\left|\E\left[y_t\right]-y_{\star}\right|\right] . \nonumber
\end{align}
    %\label{pf:regret-decompose}
\end{subequations}
Inequality (\ref{pf-regret-a-sketch}) comes from the following case discussion on $\min\{y_t,D_t\}-\min \left\{y_{\star}, D_t\right\}$: 
% of the following cases further decompose the difference of $\min\{y_t,D_t\}-\min \left\{y_*, D_t\right\}$:

\textbf{Case 1: }$D_t>y_t$: In this case, $\min \left\{y_t, D_t\right\}-\min \left\{y_{\star}, D_t\right\} = y_t - \min \left\{y_{\star}, D_t\right\} \ge y_t - y_{\star}$. Then $\mathbb{E} [ h(y_t -y_{\star}) - (h+p)(\min \left\{y_t, D_t\right\}-\min \left\{y_{\star}, D_t\right\})]  \le -p \mathbb{E}[y_t - y_{\star} ] %\textcolor{blue}{\le p \mathbb{E}[\left|y_t - y^*\right|]}\\
     = -p \mathbb{E}[\mathbb{E}[y_t] - y_{\star} ] 
     \le p \mathbb{E}[\left|\mathbb{E}[y_t] - y_{\star}\right|]$.
% \begin{equation*}
%     \begin{aligned}
%     \mathbb{E} [ h(y_t -y_*) - (h+p)(\min \left\{y_t, D_t\right\}-\min \left\{y_*, D_t\right\})] & \le -p \mathbb{E}[y_t - y^* ] \\ %\textcolor{blue}{\le p \mathbb{E}[\left|y_t - y^*\right|]}\\
%     & = -p \mathbb{E}[\mathbb{E}[y_t] - y^* ] \\
%     & \le p \mathbb{E}[\left|\mathbb{E}[y_t] - y^*\right|].
%     \end{aligned}
% \end{equation*}

\textbf{Case 2: }$D_t\le y_t$: In this case $\min \left\{y_t, D_t\right\}-\min \left\{y_{\star}, D_t\right\} = D_t - \min \left\{y_{\star}, D_t\right\} \ge 0$. Similarly, $\mathbb{E} [ h(y_t -y_{\star}) - (h+p)(\min \left\{y_t, D_t\right\}-\min \left\{y_{\star}, D_t\right\})] \le h \mathbb{E}[y_t - y_{\star} ] 
    = h \mathbb{E}[\mathbb{E}[y_t] - y_{\star} ] 
     \le h \mathbb{E}[\left|\mathbb{E}[y_t] - y_{\star}\right|]$.
% \begin{equation*}
%     \begin{aligned}
%     \mathbb{E} [ h(y_t -y_*) - (h+p)(\min \left\{y_t, D_t\right\}-\min \left\{y_*, D_t\right\})] &\le h \mathbb{E}[y_t - y^* ] \\
%     &= h \mathbb{E}[\mathbb{E}[y_t] - y^* ] \\
%     & \le h \mathbb{E}[\left|\mathbb{E}[y_t] - y^*\right|].
%     \end{aligned}
% \end{equation*}

Altogether, we show that regret analysis can be transformed into the convergence analysis of the posterior parameter. 
\subsection{\textbf{Putting All Together}} \label{subsec: everything}
Denote $C_0 = \max\{h,p\}\left(-\ln (\frac{h}{p+h})\right)^{1 / k}\left(\overline{D}^k+\frac{2}{\theta_{\star}}\right)\frac{1}{k} \cdot \min \left \{L, \frac{1}{\theta_*} \right \}^{1/k - 1} \sqrt{2\ln{\left(\frac{T}{\delta}\right)}}$.

Altogether, we have
\begin{equation*}
    \begin{aligned}    \operatorname{Regret(T,\theta_{\star})} 
     & \le \max\{h,p\} \cdot \sum_{t=1}^T \E \left[\left|\E\left[y_t\right]-y_{\star}\right|\right] \qquad \text{(Lipchitz Continuity)} \\
     & \le \max\{h,p\} \cdot \left(-\ln (\frac{h}{p+h})\right)^{1 / k} \sum_{t=1}^T \left(\overline{D}^k+\frac{2}{\theta_{\star}}\right)\sqrt{2\ln{\left(\frac{T}{\delta}\right)}}\sqrt{\frac{t}{\left(\alpha_t-1\right)^2}} \qquad \text{(Analysis of
$\left|\E\left[y_t\right]-y_{\star}\right|$)} \\
& \le C_0 \cdot \left(512 \left(1-\exp\{-\theta_{\star} L^k\}\right)^{-3} \cdot \ln\left( \frac{T}{\delta}\right)^{\frac{3}{2}} + 4\left(1-\exp\{-\theta_{\star} L^k\}\right)^{-1}\sqrt{T}\right) \qquad \text{($y_t$ lower bound)}.
      \end{aligned}
\end{equation*}

As shown, the three inequalities correspond precisely to the three key steps in Section \ref{subsubsec:lower bound action}, Section \ref{subsubsec:confidence}, and Section \ref{subsubsec:Lipchitz regret}. 

 \section{Extension beyond Weibull Distribution}
In this section, we provide analysis for Bayesian regret for arbitrary demand distribution. We first denote $\hat{F}_t$ to be estimated CDF of demand and $F_{\star}$ be the true demand CDF for unknown arbitrary demand distribution. 
\begin{assumption}
\label{ass:real}
    There exists a parameter $\theta_{\star} \in \Theta$ such that $F_{\theta_{\star}} = F_{\star}$.
\end{assumption}
\begin{assumption}
\label{ass:lip parameter}
    For the estimator $\hat{\theta}$ of the unknown parameter $\theta_{\star}$, we assume     
    $|\hat{\theta}-\theta_{\star}|\leq C_1 \cdot \|F_{\hat{\theta}}-F_{\star}\|_\infty$,
\end{assumption}
\begin{assumption}
\label{ass:cost function theta}
The newsvendor cost function satisfies $|g_{\theta_1}(y, D) - g_{\theta_2}(y, D)| \leq C_2 \cdot |\theta_1 -\theta_2|$.
\end{assumption}

Combing Assumption \ref{ass:lip parameter} with Assumption \ref{ass:cost function theta}, we can get $|g_{\theta_1}(y, D) - g_{\theta_2}(y, D)| \leq C_1C_2 \cdot \|F_{\theta_1} -F_{\theta_2}\|_\infty$.

\subsection{Kaplan-Meier Estimator and Confidence Interval}
\begin{definition}\cite{huh2011adaptive}
\label{km estimator}
\begin{align*}
     1-\hat{F}_t(x)=\prod_{s: Y_{(s)} \leqslant x}\left(\frac{t-s}{t-s+1}\right)^{\delta_{(s)}}
\end{align*}
   
\end{definition}
Notice that for $1 \le s \le t \le T$,  $Y_{(s)}$ and $\delta_{(s)}$ denote order statistics. Specifically, consider $T$ observations $\left\{\left(Y_1, \delta_1\right), \ldots,\left(Y_T, \delta_T\right)\right\}$, some of which are possibly censored (i.e., with $\delta_t=0$). To construct the KM estimator, we order the $T$ observations from smallest to largest according to the value of $Y_t$ 's. Ties are broken by placing uncensored observations ($\delta_t=1$) before censored ones ($\delta_t=0$).
\begin{definition}[Greedy Plug-in Estimator]
\label{def:theta_hat}
    At each time $t$, we define a greedy parameter estimate $\hat{\theta}_t$ that minimizes the distance to the current empirical distribution $\hat{F}_t$ under KM estimator  , and the corresponding plug-in estimator $F_{\hat{\theta}_t}$ such that $\|F_{\hat{\theta}_t} - \hat{F}_t\|_{\infty} \le \|F_{\theta_{\star}} - \hat{F}_t\|_{\infty}$.
\end{definition}

\begin{lemma}\cite{bitouze1999dvoretzky}
According to Theorem 1 of \cite{bitouze1999dvoretzky} we have, 
for each $t \in [T]$ and     let $\hat{F}_t$ be the Kaplan-Meier estimator of the distribution function $F$. And $G$ is the censoring distribution function. There exists an absolute constant $C$ such that, for any positive $\lambda$,

$$
\mathbb{P}\left(\sqrt{t}\left\|(1-G)\left(\hat{F}_t-F\right)\right\|_{\infty}>\lambda\right) \leqslant 2.5 \mathrm{e}^{-2 \lambda^2+C \lambda} .
$$ 
\end{lemma}

\begin{lemma}
\label{lemma:km confidence}
For each $t \in [T]$ and arbitrary x and $0 <G(x) <1$. and set $\epsilon=\sqrt{\frac{1}{2 t} \ln \left(\frac{1}{\delta}\right)}$, we have 
 \begin{align*}
   \|\hat{F}_t(x)-F(x)\|_{\infty} \leq \frac{\epsilon}{1-G(x)}.
\end{align*}
\end{lemma}

\subsection{Regret Analysis}

\begin{theorem}[Bayesian Regret for general demand]
\label{thm:general bayesian regret_bound}
Under Assumptions~\ref{ass:real} and~\ref{ass:lip parameter}, and with the KM-based plug-in estimator defined in Definition~\ref{def:theta_hat}, the Bayesian regret of the policy $\pi$ over $T$ periods satisfies the following bound with probability at least $1 - \delta$:
\begin{align*}
\operatorname{BayeRegret}(T, \pi) 
&\le \frac{16L(h+p)}{1 - G_{\max}} \sqrt{T \ln(1/\delta)},
\end{align*}
where $G_{\max} < 1$ is an upper bound on the CDF $G(x)$ over the demand support.
\end{theorem}

The complete proof for Theorem \ref{thm:general bayesian regret_bound} is in Appendix \ref{appendix:proof for general bayesian regret}. We provide the three key steps here. 

\paragraph{Step 1: Establishing the Plug-in Estimator Properties} 

By definition \ref{def:theta_hat}, Lemma \ref{lemma:km confidence} and assumption \ref{ass:real}, we have $\|F_{\hat{\theta}_t}-\hat{F}_t\|_{\infty} \leq\|F_{\theta_{\star}}-\hat{F}_t\|_{\infty}=\|F_{\star}-\hat{F}_t\|_{\infty}$.
% \begin{align*}
% \left|F_{\hat{\theta}_t}-\hat{F}_t\right| \leq\left|F_{\theta_{\star}}-\hat{F}_t\right|=\left|F_{\star}-\hat{F}_t\right|.
% \end{align*}
This gives us $\|F_{\hat{\theta}_t}-F_{\star} \|_{\infty}\leq 2\|\hat{F}_t-F_{\star}\|_{\infty}$.
% \begin{align*}
% F_{\hat{\theta}_t}-F_{\star} \leq 2\left|\hat{F}_t-F_{\star}\right|.
% \end{align*}
Therefore $|\hat{\theta}_t-\theta_{\star} |\leq L\|F_{\hat{\theta}_t}-F_{\star}\|_{\infty}$.
% \begin{align*}
% \hat{\theta}_t-\theta_{\star} \leq L\left|F_{\hat{\theta}_t}-F_{\star}\right|.
% \end{align*}
We define cost functions $g_{\hat{\theta}_t}(y_t, D_t)$, $\hat{g}(y_t, D_t)$, and $g_{\theta_{\star}}(y_t, D_t)$ for the estimator, KM estimator $\hat{\theta_t}$, and true parameter respectively. This yields $\left|g_{\hat{\theta}_t}(y_t, D_t)-\hat{g}(y_t, D_t)\right| \leq\left|g_{\theta_{\star}}(y_{\star}, D_t)-\hat{g}(y_t, D_t)\right|=\left|g(y_{\star}, D_t)-\hat{g}(y_t, D_t)\right|.$  According to \cite{pilz1991bayesian}, the Bayesian estimator is the minimum mean square error estimator of posterior mean. We use KM estimator to upper bound the Bayesian regret per-period.

\paragraph{Step 2: UCB Regret Decomposition} We decompose the per-period Bayesian regret  $g(y_t, D_t) - g(y_\star, D_t)$  as $[g(y_t, D_t) - U_t(y_\star)] + [U_t(y_t^{km}) - g(y_\star, D_t)]$ , here we define an upper confidence bound sequence in \cite{russo2014learning}
 $U_t(y) = \hat{g}(y, D_t)$ is the empirical cost based on the KM estimator and $y_t^{km} = \arg\min_y U_t(y)$ is the KM-optimal decision. 

\paragraph{Step 3: Putting All Together}\begin{lemma}(Lipschitz Property)
The newsvendor cost function satisfies $|g_{\theta_1}(y, D) - g_{\theta_2}(y, D)| \leq (h+p) \cdot \|F_{\theta_1} - F_{\theta_2}\|_{\infty}$.
\end{lemma}

  We exploit the Lipschitz lemma to convert the KM estimation error $|\hat{F}_t(x) - F_\star(x)| \leq \frac{\epsilon_t}{1-G(x)}$ directly into bounds on cost function differences. For Term I, this gives $g(y_t, D_t) - U_t(y_\star) \leq \frac{L_1 \epsilon_t}{1-G(y_\star)}$, and for Term II, we get $U_t(y_t^{km}) - g(y_\star, D_t) \leq \frac{L_2 \epsilon_t}{1-G_{\max}}$. Assuming that ($G(x) \leq G_{\max} < 1$. Combining both terms and summing over time with $\epsilon_t = \sqrt{\frac{\ln(1/\delta)}{2t}}$ yields the final regret bound of $ \frac{16L(h+p)}{1 - G_{\max}} \sqrt{T \ln(T)}$.

\section{Numerical Experiments}
\label{sec:numerical}

We conduct numerical experiments to evaluate the performance of Thompson Sampling (\texttt{TS}) in the repeated newsvendor problem, comparing it against three benchmarks. The first is the phased-UCB algorithm \citep{agrawal2019learning}, which updates the confidence interval at the start of each epoch (\texttt{UCB}). The second is the non-parametric adaptive policy from \cite{huh2009nonparametric}, which dynamically adjusts orders over time (\texttt{OCO}). The third is the deterministic myopic policy from \cite{besbes2022exploration}, which optimizes single-period inventory decisions without accounting for future learning. To examine the impact of different service levels, defined as \( \gamma = \frac{p}{p+h} \), we fix \( p = 1 \) and vary \( h \) to achieve service levels of 50\%, 90\%, and 98\%. All policies are tested on a common problem instance with prior parameters \( \alpha_0 = \beta_0 = 4 \) and horizon \( T = 600 \). Each algorithm is run over 100 independent trials, and we report average cumulative regret.Results are shown in two plots: (1) a comparison of \texttt{TS}, \texttt{UCB}, and \texttt{OCO} (Figure~\ref{fig:1}), and (2) a comparison of \texttt{TS} and the myopic policy against the optimal cost (Figure~\ref{fig:2}). Across all service levels, \texttt{TS} consistently outperforms the benchmarks and converges faster than the myopic policy, highlighting its effectiveness in balancing exploration and exploitation.

\begin{figure}[htb]
    \centering
    \includegraphics[width=\linewidth]{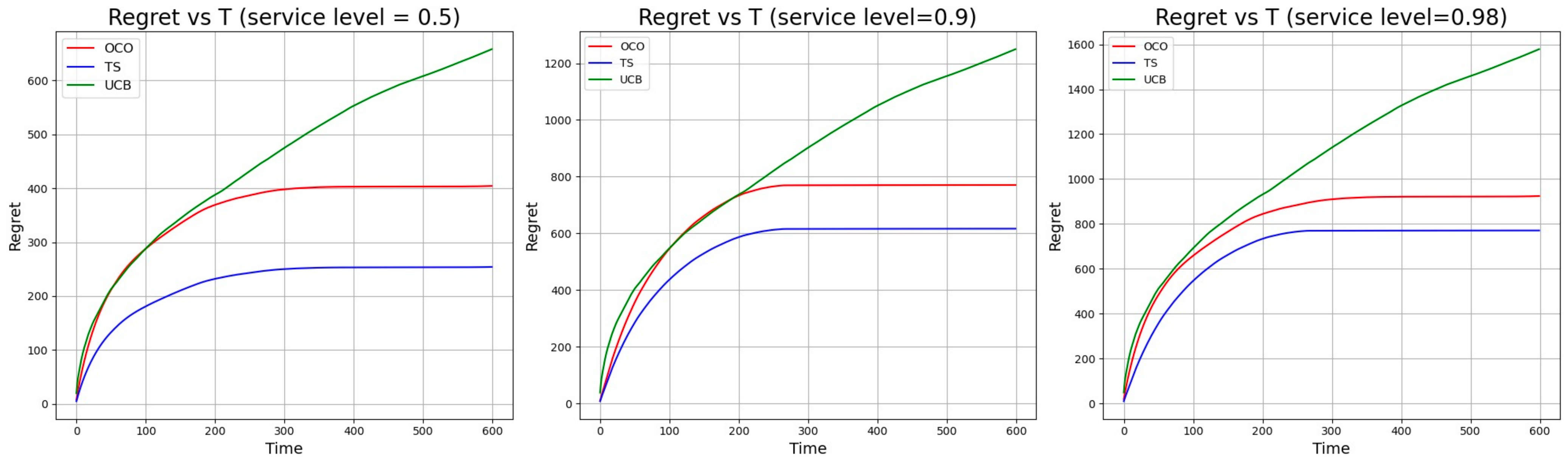}
    \caption{Compare TS with OCO and UCB}
    \label{fig:1}
\end{figure}

\begin{figure}[htb]
    \centering
    \includegraphics[width=\linewidth]{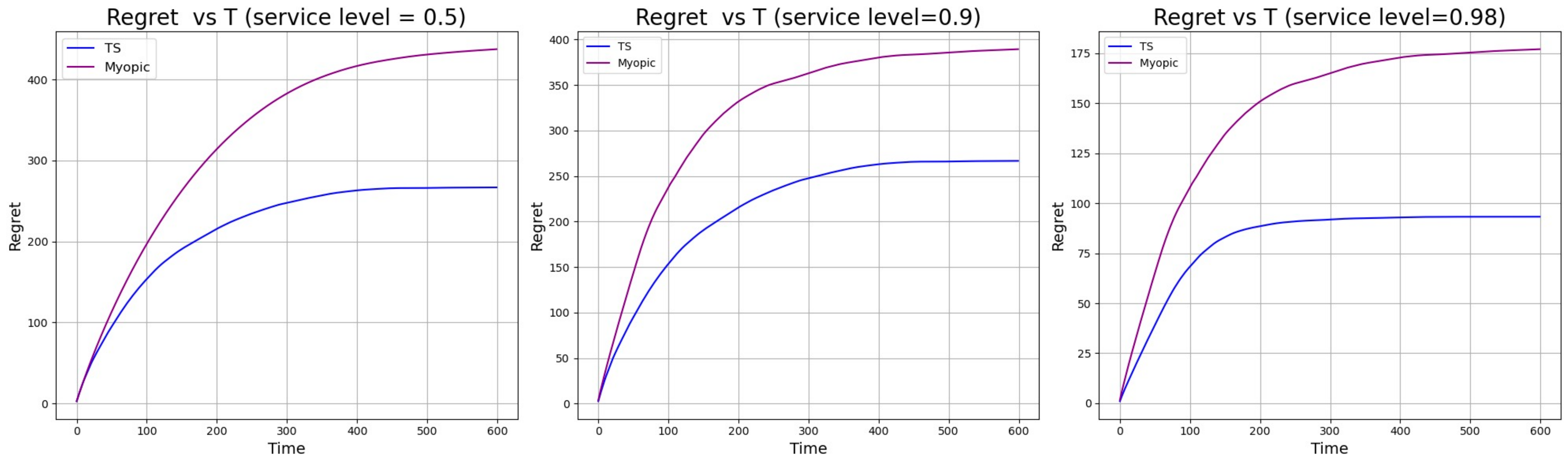}
    \caption{Compare TS with Myopic Policy}
    \label{fig:2}
\end{figure}

\newpage
\section{Conclusions}
\label{sec: conclusion}
We present the first systematic study on applying Thompson Sampling (TS) to the repeated newsvendor problem and provide an initial exploration of how our analytical framework can be extended to broader online learning problems with censored feedback. We establish frequentist regret bounds and offer insights into how TS automatically balances the trade-off between ``large exploration'' and ``optimal exploitation.'' Our analysis follows three key steps, which naturally generalize to broader settings. This work opens up a range of compelling research directions. 

We extend the newsvendor problem to general online learning with censored feedback in Appendix \ref{sec:extension}. We illustrate how our analytical framework naturally extends to broader settings of online learning with censored feedback, making it applicable to a wide range of problems where the feedback is censored.This extension requires only that the regret is Lipschitz continuous with respect to actions and that the relationship between the optimal action and the underlying parameter is continuous and monotone. Lastly, we conduct additional experiments in Appendix \ref{app:additonal experiment}.

% A key avenue for future exploration is extending regret analysis to broader online learning environments with censored feedback, particularly by relaxing the Weibull demand assumption to develop a more flexible and general framework. Additionally, applying TS to broader economic settings—such as auctions, dynamic pricing, and real-time resource allocation—presents exciting opportunities, as willingness-to-pay observations are often censored in these contexts. Advancing research in these areas has the potential to enhance decision-making under uncertainty, fostering more robust, efficient, and adaptive learning mechanisms for complex real-world problems.
\newpage

\bibliographystyle{plainnat}
\bibliography{ref.bib}  % Replace with your .bib file name

\appendix
\section{Appendix}
\subsection{Extensions to Online Learning with Censored Feedback}
\label{sec:extension}

In this section, we extend the regret analysis of TS for the repeated newsvendor problem to a broader class of online learning algorithms. We consider a setting where the demand $D_t$ in period $t$ is drawn from Weibull distribution parameterized by $\theta_{\star}$. The decision-maker selects an action $A_t$ in each period, resulting in an observed feedback of $\min\{D_t, A_t\}$. The loss incurred in each period is defined as $l(A_t) = \min\{D_t, A_t\}$.
%A_t >0 
The cumulative regret over $T$ periods is defined as $\operatorname{Regret}(T, \theta_{\star}) = \sum_{t=1}^T \left( l(A_t) - l(A_{\star}) \right)$, where $A_{\star}$ denotes the optimal action that minimizes the expected loss, given by $A_{\star} = \arg\min_{A} \mathbb{E}_{D \sim \theta_{\star}}[l(A)]$.

\subsubsection{Key Assumptions and Results}
We make the following assumptions to facilitate the analysis:

\begin{assumption}[Lipschitz Continuity of Regret]\label{assumption:regret}
    The regret function is Lipschitz continuous with respect to the action, allowing it to be decomposed as:

\[
\operatorname{Regret}(T, \theta_{\star}) \leq C_1 \sum_{t=1}^T \mathbb{E} \left[ \left| \mathbb{E}[A_t] - A_{\star} \right| \right],
\]

where $C_1$ is a positive constant.
\end{assumption}
 Clearly, Assumption \ref{assumption:regret} is precisely the conclusion of key step 1 (see Section \ref{subsubsec:Lipchitz regret}) in our earlier analysis of the repeated newsvendor problem. Consequently, once this assumption is satisfied, no further model requirements are needed to validate the conclusions drawn in key step 1.
\begin{assumption}[Lipschitz Continuity and Monotonicity of $A_t$]\label{assumption:action}
    $A_t$ is non-decreasing with respect to $\frac{1}{\theta_t}$, and the deviation of the expected action from the optimal action is proportional to the estimation error of the parameter $\theta_{\star}$, such that:
\begin{align}\label{eq:lipchitz action}
    \left| \mathbb{E}[A_t] - A_{\star} \right| \le C_2 \left| \mathbb{E}[\frac{1}{\theta_t}] - \frac{1}{\theta_{\star}} \right|,
\end{align}
where $\theta_t$ is the parameter estimate at time $t$, and $C_2$ is a positive constant.
\end{assumption}
Assumption \ref{assumption:action} is satisfied in the repeated newsvendor model. Specifically, Lemma \ref{lem:y_t-and-y_star} shows that the optimal action in the newsvendor problem is given by $y_t(\theta_t)=\frac{1}{\theta_t}\left(-\ln(\frac{h}{p+h}\right)^{1/k}$. Let us show how the the conclusions in  key step 2 and key step 3 hold under this assumption.For the key step 2, the Lipschitz continuity assumption in \eqref{eq:lipchitz action} directly leads to \eqref{eq:lipchitz action 2}. Additionally, Lemma \ref{beta/alpha-1} provides a generic estimation result for censored feedback under the Weibull distribution that is independent of the loss function or algorithm in use. Consequently, the conclusion of key step 2 (Section \ref{subsubsec:confidence}) holds under Assumption \ref{assumption:action}. For key step 3, we observe that the lower bounds for  $\frac{1}{\theta_t}$ (as shown in inequalities \eqref{lb-yt-b} to \eqref{lb-yt-d}) are general results for censored feedback under the Weibull distribution and do not depend on the loss function or the specific algorithm used. Therefore, as long as the positive function $y_t$ is a monotone in $\frac{1}{\theta_t}$, a uniform lower bound for
$y_t$  is guaranteed.
%\paragraph{Assumption 3 (Uniform Lower Bound on Actions)}
%There exists a constant $K > 0$ such that, with high probability, $A_t \geq K$ for all $t$.
By synthesizing the above analysis on how the conclusions of all three key steps hold, we establish the following theorem on cumulative regret:
\begin{theorem}[Regret of TS for general online learining with censored feedback]Under Assumption \ref{assumption:regret} and Assumption \ref{assumption:action}, we have that 
    \[
\operatorname{Regret}(T, \theta_{\star}) \leq O\left( C_3 \ln(T) \sqrt{T} \right),
\]
where $C_3$ is a positive constant that depends on $C_1$, $C_2$, and the distribution parameters.
\end{theorem}

This establishes the $\sqrt{T}-$regret for the general online learning model we considered in this section.

\subsection{Proof for Lemma \ref{lem:demand-range}}
\label{appendix: lemma demand range}
\begin{proof}
Since $D_t \sim \operatorname{Weibull}(\theta_{\star}
    )$, the cumulative distribution function for demand $D_t$ is indicated as $F_{D_t}(x)=1-e^{-\theta_{\star}x^k}$. Then we have
    \begin{align*}
          \mathbb{P}\left(D_t < \underline{D}\right)=1 - e^{-\theta_{\star}\underline{D}^k}  \le  \frac{\delta}{2T}, \qquad \mathbb{P}\left(D_t > \Bar{D}\right)=e^{-\theta_{\star}\Bar{D}^k} \le \frac{\delta}{2T}. 
    \end{align*}
    Choose appropriate $\underline{D},\overline{D}$ that satisfy above two inequalities and we obtain the lemma.
\end{proof}
\subsection{Proof for Lemma \ref{beta/alpha-1}}
\label{sec:appendix-lemma beta/alpha-1}
\begin{proof}

The proof largely follows Lemma B2, B3 in \cite{chuang2023bayesian}. We denote $H = \left\{H_t \right\}$ the natural filtration generated by the right-censored sales data, i.e $H_t=\sigma  \left\{(Y_i,\delta_i) : i \leq t\right\} $, where $Y_t =D_t \wedge y_t$ and $\delta_t =1 \left[D_t <  y_t \right]$.

According to the proof of Lemma B2 and B3 in \cite{chuang2023bayesian}, we have

\begin{align*}
     N_t =\sum_{i=0}^{t-1}\left(Y_i^k-\mathbb{E}\left[Y_i^k \mid \mathcal{H}_{i-1}\right]\right),\qquad  M_t =\sum_{i=0}^{t-1}\left(\delta_i-\mathbb{E}\left[\delta_i \mid \mathcal{H}_{i-1}\right]\right)-1.
\end{align*}
$\left\{M_t\right\}$ and $\left\{N_t\right\}$ are zero-mean martingales given that the true unknown parameter is $\theta_{\star} \in \mathbb{R}_{+}$. We define $A_t=\sum_{i=0}^{t-1}\left(1-e^{-\theta_{\star} y_i^k}\right)$ then

Then we have, 
\begin{align*}
     \frac{\beta_t}{\alpha_t-1} -\frac{1}{\theta_{\star}}  &=\frac{1}{\theta_{\star}}\left(\frac{A_t+\theta_{\star} N_t}{A_t+ M_t-1}-1\right)\\
    &=\frac{1}{\theta_{\star}}\left(\frac{A_t+\theta_{\star} N_t-A_t-M_t+1}{A_t+ M_t-1}\right)\\
    &=\frac{1}{\theta_{\star}}\left(\frac{\theta_{\star} N_t-M_t+1}{A_t+ M_t-1}\right)\\
    &=\frac{N_t-\frac{1}{\theta_{\star}}\left(M_t-1\right)}{\alpha_t-1}.
\end{align*}

From \cite{chuang2023bayesian}, we have 
\begin{align*}
    N_t =\sum_{i=0}^{t-1}\left(Y_i^k-\mathbb{E}\left[Y_i^k \mid \mathcal{H}_{i-1}\right]\right),\qquad M_t-1 =\sum_{i=0}^{t-1}\left(\delta_i-\mathbb{E}\left[\delta_i \mid \mathcal{H}_{i-1}\right]\right)-1
\end{align*}
Therefore,
\begin{align*}
   N_t-\frac{1}{\theta_{\star}}\left(M_t-1\right)=\sum_{i=0}^{t-1}\left(Y_i^k-\frac{\left(\delta_i-1\right)}{\theta_{\star}}-\mathbb{E}_{\theta_{\star}}\left[\left(Y_i^k-\frac{\delta_i}{\theta_{\star}}\right) \mid \mathcal{H}_{i-1}\right]\right)
\end{align*} is a martingale values and satisfy 
\begin{align*}
    & Y_i^k- \frac{\left(\delta_i-1\right)}{\theta_{\star}} \leq \min\{D_i,y_i\}^k+\frac{2}{\theta_{\star}} \leq \overline{D}^k+\frac{2}{\theta_{\star}}
\end{align*}
  
Applying the Azuma–Hoeffding inequality, for $t \in [T]$,  with probability $1-\frac{1}{t^2}$
\begin{align*}
    \mathbb{P}\left( \left|\frac{\beta_t}{\alpha_t-1}- \frac{1}{\theta_{\star}}\right| =\frac{N_t-\frac{1}{\theta_{\star}}\left(M_t-1\right)}{\alpha_t-1} \geq \epsilon_t \right)&=\mathbb{P}\left( \left|  \frac{\beta_t}{\alpha_t-1}- \frac{1}{\theta_{\star}}\right| =N_t-\frac{1}{\theta_{\star}}\left(M_t -1\right)\geq \left(\alpha_t -1 \right)\epsilon_t \right)\\
 & \leq 2\exp{\left(\frac{- \epsilon_t^2 \cdot  \left(\alpha_t-1\right)^2 }{t \cdot\left(\overline{D}^k+\frac{2}{\theta_{\star}}\right)^2}  \right)}
 \end{align*}
 Plug in $\epsilon_t= \sqrt{\ln{\left(\frac{2t^2}{\delta}\right)}}\left(\overline{D}^k+\frac{2}{\theta_{\star}}\right)\sqrt{\frac{t}{\left(\alpha_t-1\right)^2}}$ then we obtain the lemma.
\end{proof} 

\subsection{Proof for Lemma \ref{lower bound of y_t}}
\label{appendix: lower bound of y_t}
\begin{proof}
    Since $\theta_t \sim \operatorname{Gamma}(\alpha_t,\beta_t)$, we have $\frac{1}{\theta_t} \sim \operatorname{InverseGamma}(\alpha_t,\beta_t)$. According to \cite{chen2014concentration} Theorem 20, we have  
    
A random variable $X$ is said to have an inverse gamma distribution if it possesses a probability density function

$$
f(x)=\frac{\beta^\alpha}{\Gamma(\alpha)} x^{-\alpha-1} \exp \left(-\frac{\beta}{x}\right), \quad x>0, \quad \alpha>0, \quad \beta>0
$$

Let $X_1, \cdots, X_t$ be i.i.d. samples of random variable $X$. By virtue of the LR method, we have obtained the following results.

$$
\begin{aligned}
 \mathbb{P}\left\{\bar{X}_n \leq z\right\} \leq\left[\left(\frac{\beta}{\alpha z}\right)^\alpha \exp \left(\frac{\alpha z-\beta}{z}\right)\right]^n \quad \text { for } 0<z \leq \frac{\beta}{\alpha}
\end{aligned}
$$
 $\forall \ t \in [T]$, We plug in $n=1, z=\frac{\beta_t}{2\alpha_t}$ and $\bar{X}_t=\frac{1}{\theta_t}$, then get 
\begin{align*}
    \mathbb{P}\left(\frac{1}{\theta_t} \leq \frac{\beta_t}{2\alpha_t}\right) \leq \left(\frac{2}{e}\right)^{\alpha_t}
\end{align*}
Then, $\forall \ t \in [T], \mathbb{P}\left(\frac{1}{\theta_t} > \frac{\beta_t}{2\alpha_t}\right) \ge 1-\left(\frac{2}{e}\right)^{\alpha_t}$
\end{proof}
\subsection{Proof for Lemma \ref{lemma:sequence a and b}}
\label{appendix: proof sequence a and b}
\begin{proof}
    The proof is straightforward. Denote
    \begin{equation*}
        \min_{i \in [n]: b_i > 0} \left \{ \frac{a_i}{b_i}\right \} = \kappa,
    \end{equation*}
    then $a_i \ge \kappa b_i$ for any $i$ such that $b_i > 0$. Hence,
    \begin{equation*}
        \frac{\sum_{i=1}^n a_i}{ \sum_{i=1}^n b_i} \ge \frac{\sum_{i=1}^n 
 \kappa b_i}{ \sum_{i=1}^n b_i} = \kappa = \min_{i \in [n]: b_i > 0} \left \{ \frac{a_i}{b_i}\right \}.
    \end{equation*}
    This completes the proof.
\end{proof}
\subsection{Proof for Lemma \ref{lem:Mt-confidence}}
\label{appendix:Mt}
\begin{proof}
    Recall that $M_t=\sum_{i=0}^{t-1}\left(\delta_i-\mathbb{E}\left[\delta_i \mid \mathcal{H}_{i-1}\right]\right)$ defined in Lemma \ref{lemma:alpha_t,beta_t} and
$M_t$ is a martingale with bounded increments (specifically, bounded by 2 ), by Azuma's inequality we have,
\begin{align*}
    \mathbb{P}\left(\left|M_t\right| \geq \epsilon\right) \leq 2 \exp \left(-\frac{\epsilon_t^2}{8 t}\right).
\end{align*}
Therefore $\mathbb{P}\left(\left|M_t\right| \geq \sqrt{8t}\ln{\left(
    \frac{2t^2}{\delta}\right)} \right) \leq \frac{\delta}{ t^2}$. 
\end{proof}
\subsection{Proof for Lemma \ref{lemma:km confidence}}
\begin{proof}
    Setting $\epsilon=\sqrt{\frac{1}{2 t} \ln \left(\frac{C}{\delta}\right)}$, then with probability at least $1-\delta$, we have
\begin{align*}
 \sup _x\left|(1-G(x))\left(\hat{F}_t(x)-F(x)\right)\right| \leq \epsilon.    
\end{align*}
Therefore, For any $x$ with $G(x)<1$, the following holds, 
\begin{align*}
   \left|\hat{F}_t(x)-F(x)\right| \leq \frac{\epsilon}{1-G(x)}.
\end{align*}
So:

\begin{align*}
 \hat{F}_t(x)-\frac{\epsilon}{1-G(x)} \leq F(x) \leq \hat{F}_t(x)+\frac{\epsilon}{1-G(x)}. 
\end{align*}
\end{proof}
\subsection{Proof for Theorem \ref{thm:general bayesian regret_bound}}
\label{appendix:proof for general bayesian regret}
\begin{proof}
According to definition \ref{def:theta_hat} and Lemma \ref{lemma:km confidence} combining with assumption \ref{ass:real}, we have
\begin{align*}
    \left|F_{\hat{\theta}_t} - \hat{F}_t\right| \le \left|F_{\theta_{\star}} - \hat{F}_t\right| = \left|F_{\star} - \hat{F}_t\right|.
\end{align*}
Then we have 
\begin{align*}
    F_{\hat{\theta}_t} - \hat{F}_{\star} \le 2 \left|\hat{F}_t-\hat{F}_{\star}\right|.
\end{align*}
Given the Lipchitz assumption (\ref{ass:lip parameter}), we have for each $t \in [T]$,
 \begin{align*}
        \hat{\theta}_t-\theta_{\star} \le L \left|F_{\hat{\theta}_t}-F_{\star}\right|.
    \end{align*}
Since the plug-in estimator $\hat{\theta}$ defined in \ref{def:theta_hat} gives $\theta_{\star}$ a Minimum mean square estimation. Referring to Algorithm \ref{alg:ts for newsvendor} and Using the MSE as risk, the Bayes estimate of the unknown parameter is simply the mean of the posterior distribution and is known as the minimum mean square error (MMSE) estimator.

Then, we define the newsvendor cost function parametrized by estimator $\hat{\theta}_t$ denoted as $g_{\hat{\theta}_t}(y_t,D_t)$. The newsvendor cost function parametrized by km estimator denoted as $\hat{g}(y_t,D_t)$. And the newsvendor cost function parametrized by the function class $\Theta$ denoted as $g_{\theta_{\star}}(y_t,D_t)$, the following inequality holds since the derivative of the newsvendor cost is Lipchitz. 
\begin{align*}
    \left| g_{\hat{\theta}_t}(y_t,D_t)- \hat{g}(y_t,D_t)\right| \le \left|g_{\theta_{\star}}\left(y_{\star}, D_{t}\right) - \hat{g}(y_t,D_t)\right| =\left|g\left(y_{\star}, D_{t}\right)- \hat{g}(y_t,D_t)\right|.
\end{align*}

According to \cite{russo2014learning}, we define $U_t(y):=\hat{g}(y_t,D_t)$ and $y_t^{km} = \arg\min_y \hat{g}(y_t,D_t)$. Then for each $t \in [T]$, we have 
The per-period Bayesian regret defined in \ref{def:bayesian regret} can be decomposed as:
\begin{align*}
g\left(y_t, D_t\right) - g(y_\star, D_t)
 \le [g\left(y_t, D_t\right) - U_t(y_\star)] + [U_t(y_t^{km}) - g(y_\star, D_t)],
\end{align*}
\paragraph{Bounding Term I: $g\left(y_t, D_t\right) - U_t(y_\star)$}

From Lemma~\ref{lemma:km confidence}, for any $x$ such that $G(x)<1$, we have
\[
| \hat{F}_t(x) - F_{\theta_\star}(x) | \le \frac{\epsilon_t}{1 - G(x)}.
\]
By Lipschitz continuity of $g_{\theta}$ in distribution as stated in Assumption~\ref{ass:lip parameter}, we get:
\[
g\left(y_t, D_t\right) - U_t(y_\star) \le |g_{\theta_\star}(y_\star, D_t) - \hat{g}(y_\star, D_t)| \le \frac{L \epsilon_t}{1 - G(y_\star)}.
\]

\paragraph{Bounding Term II: $U_t(y_t^{km}) - g(y_\star, D_t)$}

Using Lemma \ref{lemma:km confidence} and the fact that cost function $U_t(y):=\hat{g}(y_t,D_t)$ is Lipschitz with respect to its first-order derivative $\hat{f}_t$. Then, 
\begin{align*}
    U_t(y_t^{km}) - g(y_\star, D_t) \le \frac{(h+p) \epsilon_t}{1 - G(y_\star)}.
\end{align*}

Combining both terms, we obtain:

\begin{align*}
g\left(y_t, D_t\right) - g(y_\star, D_t)
 \le \frac{L \epsilon_t}{1 - G(y_\star)} + \frac{2(h+p) \epsilon_t}{1 - G(y_t)}.
\end{align*}

Summing over $t = 1$ to $T$ and taking expectations:
\begin{align*}
\operatorname{Regret(T, \pi,\theta_{\star})}&=\mathbb{E}\left[\sum_{t=1}^Tg\left(y_t, D_t\right) - \sum_{t=1}^T g\left(y_{\star}, D_{t}\right)\mid 
 \theta_{\star}\right] \\
&\le \sum_{t=1}^T \mathbb{E} \left[\frac{L \epsilon_t}{1 - G(y_\star)} + \frac{2(h+p)\epsilon_t}{1 - G(y_t)}\right].
\end{align*}

There exists a constant $G_{\max} < 1$ such that $G(x) \leq G_{\max}$ for all $x$ in the support of the demand distribution. 

From  Lemma \ref{lemma:km confidence}, we have:
$$\epsilon_t = \sqrt{\frac{1}{2t} \ln\left(\frac{1}{\delta}\right)}$$

Substituting this into our regret bound:

\begin{align*}
\operatorname{Regret}(T, \pi,\theta_{\star}) &\le \sum_{t=1}^T \mathbb{E} \left[\frac{L \epsilon_t}{1 - G(y_\star)} + \frac{2(h+p)\epsilon_t}{1 - G_{\max}}\right] \\
&= \sum_{t=1}^T \left(\frac{L}{1 - G(y_\star)} + \frac{2(h+p)}{1 - G_{\max}}\right) \sqrt{\frac{1}{2t} \ln\left(\frac{1}{\delta}\right)} \\
&\le  \frac{16L(h+p)}{1 - G_{\max}}  \sqrt{\frac{\ln(1/\delta)}{2}} \sum_{t=1}^T \frac{1}{\sqrt{t}}\\
& \le  \frac{16L(h+p)}{1 - G_{\max}}  \sqrt{\ln(T)T} 
\end{align*}

% Let $C' = \frac{16L(h+p)}{1 - G_{\max}} $. Then:

% \begin{align*}
% \operatorname{Regret}(T, \pi,\theta_{\star}) &\le \left(\frac{L_{1}}{1 - G(y_\star)} + \frac{2L_2}{1 - G_{\max}}\right) \sqrt{\frac{\ln(C/\delta)}{2}} \cdot 2\sqrt{T} \\
% &= 2\sqrt{2} \left(\frac{L_{1}}{1 - G(y_\star)} + \frac{2L_2}{1 - G_{\max}}\right) \sqrt{T \ln(C/\delta)}
% \end{align*}

\end{proof}
\subsection{Auxiliary Lemmas}
\label{appendix: proof for lemma: alpha_t-1}
\begin{lemma}
    \label{lemma: alpha_t-1 }
    Denote 
$$T_0 = 64\left(1-\exp\{-\theta_{\star} L^k\}\right)^{-2}\ln{\frac{T}{\delta}},$$ Therefore,
    \begin{equation*}
    \alpha_t - 1 \ge \begin{cases}
        \alpha_0 - 1 & t \le T_0,\\
        \frac{1}{2} t\left(1-\exp\{-\theta_{\star} L^k\}\right) & t > T_0.
    \end{cases}
\end{equation*}
\end{lemma}
\begin{proof}
    Given above $\alpha_t$ is defined as $\alpha_t = \alpha_0 + \sum_{i=1}^{t} \delta_i$. Given that $\alpha_0 \ge 2$, it follows that $\alpha_t \ge \alpha_0 \ge 2$, and thus $\alpha_t - 1 > 0$ for all $t$.

To facilitate our analysis, we define the following high-probability events:

\begin{equation*}
    \begin{aligned}
    & \xi^{(1)}_{t} = \{\underline{D}\le D_t\le \overline{D} \},\quad \xi^{(2)}_{t} = \left\{\left|\frac{\beta_t}{\alpha_t-1}- \frac{1}{\theta_{\star}}\right| \leq \sqrt{\ln{\left(\frac{2t^2}{\delta}\right)}}\left(\overline{D}^k+\frac{2}{\theta_{\star}}\right)\sqrt{\frac{t}{\left(\alpha_t-1\right)^2}}\right \} \\
    & \xi^{(3)}_{t}= \left \{\left|M_t\right| \leq \sqrt{8t}\ln{\left(
    \frac{2t^2}{\delta}\right)} \right \},\quad \xi^{(4)}_{t}= \left\{\frac{1}{\theta_t} > \frac{\beta_t}{2\alpha_t} \right \},
    \end{aligned}
\end{equation*}
and $\xi^{(1)} = \cap_{t=1}^T \xi^{(1)}_{t}$, $\xi^{(2)} = \cap_{t=1}^T \xi^{(2)}_{t}$, $\xi^{(3)} = \cap_{t=1}^T \xi^{(3)}_{t}$, $\xi^{(4)} = \cap_{t=1}^T \xi^{(4)}_{t}$, $\xi = \cap_{i=1}^4 \xi^{(i)}$. Condition on event $\xi$, we have for all $t\in[T]$, 
\begin{subequations}
\begin{align}
     \alpha_t-1&=\alpha_0+\sum_{i=0}^{t-1}\left(1-e^{-\theta_{\star} y_i^k}\right)+M_t-1 \label{eq:sket-pf-bound-yt-alphat-a}\\
     &\ge \alpha_0-1+\left(t-1\right)\left(1-\exp\{-\theta_{\star} L^k\right)+M_t \label{eq:sket-pf-bound-yt-alphat-b}\\
    & \ge t\left(1-\exp\{-\theta_{\star} L^k\}\right)+\alpha_0-1-\left(1-\exp\{-\theta_{\star} L\}^k\right)-\sqrt{8t}\ln{\left(
    \frac{2t^2}{\delta}\right)} \label{eq:sket-pf-bound-yt-alphat-c} \\
    & \ge t\left(1-\exp\{-\theta_{\star} L^k\}\right)-\sqrt{8t}\ln{\left(
    \frac{2t^2}{\delta}\right)} \nonumber . 
\end{align}
\label{eq:sket-pf-bound-yt-alphat}
\end{subequations}
 (\ref{eq:sket-pf-bound-yt-alphat-a}) is derived from Lemma \ref{lemma:alpha_t,beta_t}. (\ref{eq:sket-pf-bound-yt-alphat-b}) comes from the fact that $y_t \ge L$ for all $t$ when the event $\xi$ holds. (\ref{eq:sket-pf-bound-yt-alphat-c}) comes from the following Lemma \ref{lem:Mt-confidence}, which is proved in Appendix \ref{appendix:Mt}. 
\begin{lemma}
    \label{lem:Mt-confidence}
    For $t \in [T]$,
    \begin{align*}
        \mathbb{P} \left(M_t \geq \sqrt{8t}\ln{\left(
    \frac{2t^2}{\delta}\right)}\right) \leq \frac{\delta}{ t^2}.
    \end{align*}
\end{lemma}

To further analyze $\alpha_t - 1$, we use the technique of truncating $T$ as follows:

Denote 
$$T_0 = 64\left(1-\exp\{-\theta_{\star} L^k\}\right)^{-2}\ln{\frac{T}{\delta}},$$
When $t > T_0$, we have
\begin{equation*}
    \alpha_t-1 \ge t\left(1-\exp\{-\theta_{\star} L^k\}\right)-\sqrt{8t}\ln{\left(
    \frac{2t^2}{\delta}\right)} > \frac{1}{2} t\left(1-\exp\{-\theta_{\star} L^k\}\right).
\end{equation*}
Therfore we have,
\begin{equation*}
    \alpha_t - 1 \ge \begin{cases}
        \alpha_0 - 1 & t \le T_0,\\
        \frac{1}{2} t\left(1-\exp\{-\theta_{\star} L^k\}\right) & t > T_0.
    \end{cases}
\end{equation*}

Finally we discuss the probability of event $\xi$.
\begin{subequations}
    \begin{align}
    \mathbb{P}(\xi) & = 1 - \sum_{i=1}^4 \mathbb{P}(\neg\xi^{(i)}) \nonumber  \\
    & = 1 - \sum_{i=1}^4 \sum_{t=1}^T \mathbb{P}(\neg\xi^{(i)}_t) \nonumber \\
    & \ge 1 - \sum_{t=1}^T \frac{\delta}{T} - \sum_{t=1}^T \frac{\delta}{t^2} - \sum_{t=1}^T \frac{\delta}{t^2}- \sum_{t=1}^T \left(\frac{2}{e} \right)^{\alpha_t} \label{eq:pf-yt-xi-bound-a} \\
    & \ge 1 - \delta - \frac{\pi^2}{6} \delta - \frac{\pi^2}{6} \delta - \delta \label{eq:pf-yt-xi-bound-b} \\
    & \ge 1 - 6 \delta.\nonumber
    \end{align}
    \label{eq:pf-yt-xi-bound}
\end{subequations}
For (\ref{eq:pf-yt-xi-bound-a}), the first term comes from Lemma \ref{lem:demand-range}, the second term comes from Lemma \ref{beta/alpha-1}. The third term comes from Lemma  \ref{lem:Mt-confidence}. The fourth term comes from Lemma \ref{lower bound of y_t}. (\ref{eq:pf-yt-xi-bound-b}) comes from the following, recall $\alpha_0 \ge \frac{\ln{\frac{T}{\delta}}}{\ln{\frac{e}{2}}}$.
    \begin{equation*}
        \sum_{t=1}^T\left(\frac{2}{e}\right)^{\alpha_t} \le \sum_{t=1}^T  \left(\frac{2}{e}\right)^{\alpha_0} = T \cdot \left(\frac{2}{e}\right)^{\alpha_0} = T \cdot e^{-\ln(e/2) \cdot \alpha_0} = T \cdot \left(\frac{\delta}{T} \right) \le \delta. \nonumber
    \end{equation*}

Consequently, with probability $\ge 1 - 6 \delta$,
\begin{equation*}
    \alpha_t - 1 \ge \begin{cases}
        \alpha_0 - 1 & t \le T_0,\\
        \frac{1}{2} t\left(1-\exp\{-\theta_{\star} L^k\}\right) & t > T_0.
    \end{cases}
\end{equation*}
\end{proof}
\newpage
 \subsection{Additional Experiments}
 \label{app:additonal experiment}
 In this section, we conduct numerical experiments for normal distributions. To assess the impact of varying service levels, defined as $ \gamma = \frac{p}{p+h}$, we set $p = 1$ and adjust $h$ to achieve service levels of 50 \%, 90\%, and 98 \%. All policies are evaluated with prior parameters \( \alpha_0 = \beta_0 = 4 \) and time horizon  \( T = 600 \). Each algorithm runs across 100 independent trials, with average cumulative regret reported.
The results are presented in two figures: (1) a performance comparison of TS, UCB, and OCO algorithms (Figure~\ref{fig:3}), and (2) a comparison of TS and the myopic policy relative to optimal cost (Figure~\ref{fig:4}). Across all service levels, TS demonstrates superior performance compared to benchmark methods and achieves faster convergence than the myopic policy, demonstrating its effectiveness in balancing exploration and exploitation.
 % In this section, we run numerical for normal distributions. To examine the impact of different service levels, defined as \( \gamma = \frac{p}{p+h} \), we fix \( p = 1 \) and vary \( h \) to achieve service levels of 50\%, 90\%, and 98\%. All policies are tested on a common problem instance with prior parameters \( \alpha_0 = \beta_0 = 4 \) and horizon \( T = 600 \). Each algorithm is run over 100 independent trials, and we report average cumulative regret.Results are shown in two plots: (1) a comparison of \texttt{TS}, \texttt{UCB}, and \texttt{OCO} (Figure~\ref{fig:3}), and (2) a comparison of \texttt{TS} and the myopic policy against the optimal cost (Figure~\ref{fig:4}). Across all service levels, \texttt{TS} consistently outperforms the benchmarks and converges faster than the myopic policy, highlighting its effectiveness in balancing exploration and exploitation.

 \begin{figure}[htb]
    \centering
    \includegraphics[width=\linewidth]{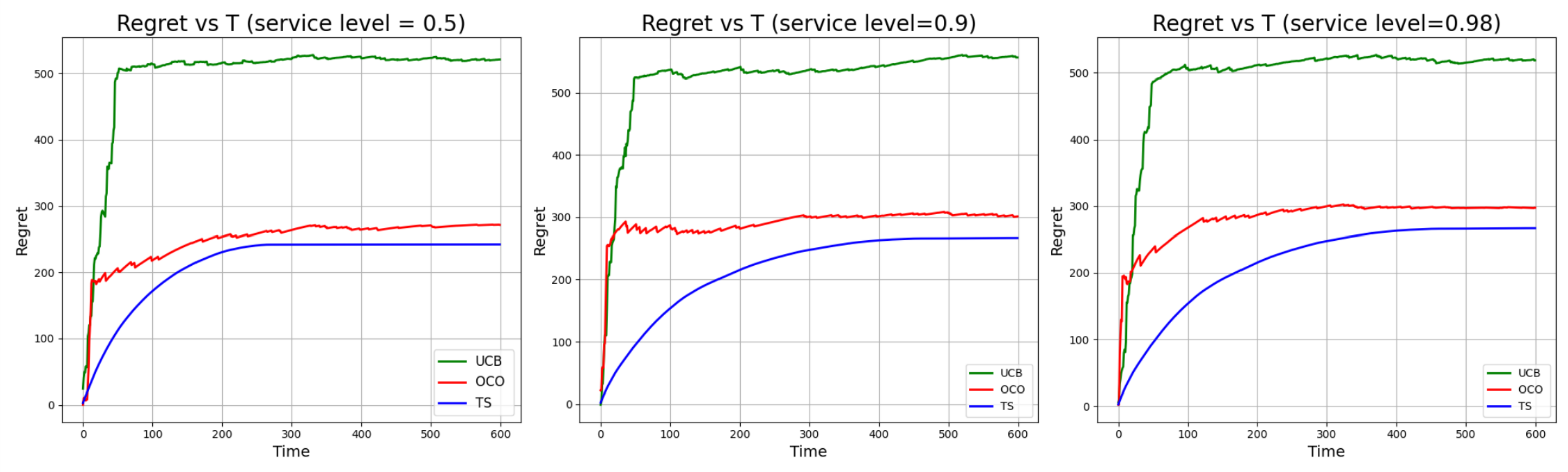}
    \caption{(Normal Distribution) Compare TS with OCO and UCB}
    \label{fig:3}
\end{figure}

\begin{figure}[htb]
    \centering
    \includegraphics[width=\linewidth]{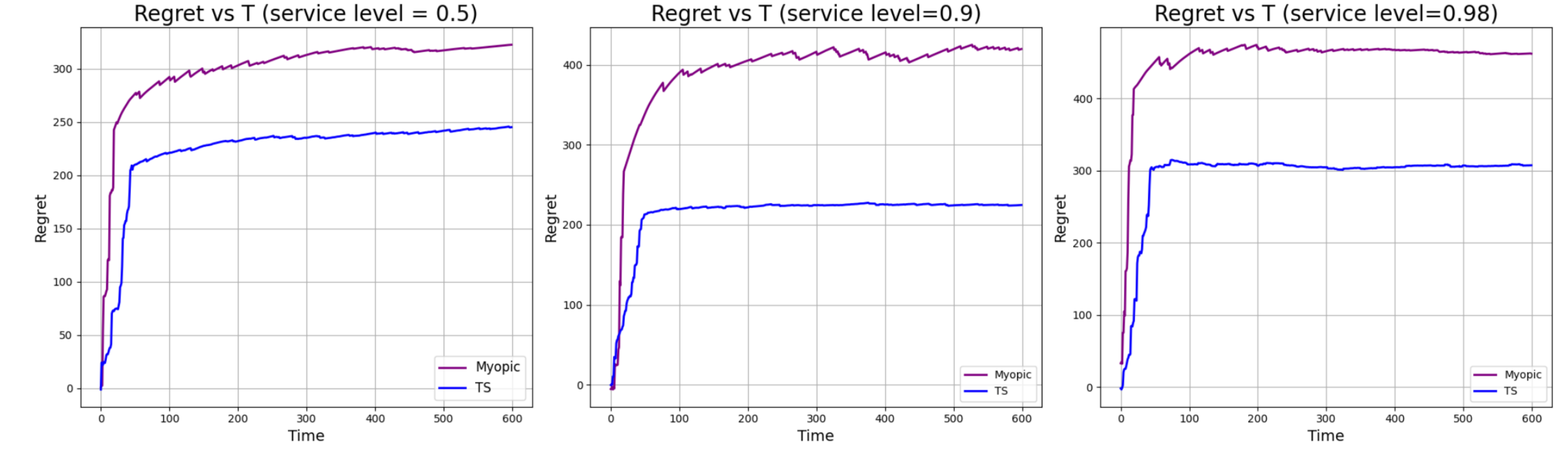}
    \caption{(Normal Distribution) Compare TS with Myopic}
    \label{fig:4}
\end{figure}
\end{document}